\documentclass[10pt,twocolumn,a4paper]{article}

\usepackage[top=0.9in,bottom=1in,left=0.65in,right=0.65in,columnsep=0.22in]{geometry}
\usepackage{times}
\usepackage{amsmath,amssymb,amsfonts}
\usepackage{amsthm}
\usepackage{algorithm}
\usepackage{algorithmic}
\usepackage{graphicx}
\usepackage{booktabs}
\usepackage{multirow}
\usepackage{xcolor}
\usepackage{hyperref}
\usepackage{parskip}
\usepackage{titlesec}
\usepackage{microtype}
\usepackage{mathtools}
\usepackage{url}
\usepackage{array}
\usepackage{tcolorbox}
\usepackage{enumitem}
\usepackage{caption}
\usepackage{tikz}
\usetikzlibrary{arrows.meta,positioning,shapes.geometric}

\hypersetup{colorlinks=true,linkcolor=blue,citecolor=blue,urlcolor=blue}

\titleformat{\section}{\normalsize\bfseries}{\thesection}{0.5em}{}
\titleformat{\subsection}{\normalsize\bfseries}{\thesubsection}{0.5em}{}
\titleformat{\subsubsection}{\small\itshape}{\thesubsubsection}{0.5em}{}
\titlespacing*{\section}{0pt}{7pt}{3pt}
\titlespacing*{\subsection}{0pt}{5pt}{2pt}
\titlespacing*{\subsubsection}{0pt}{3pt}{1pt}

\newtheorem{theorem}{Theorem}[section]
\newtheorem{lemma}[theorem]{Lemma}
\newtheorem{definition}[theorem]{Definition}
\newtheorem{corollary}[theorem]{Corollary}
\newtheorem{assumption}{Assumption}
\newtheorem{proposition}[theorem]{Proposition}
\newtheorem{remark}[theorem]{Remark}

\definecolor{boxblue}{RGB}{235,243,251}
\definecolor{accentblue}{RGB}{46,117,182}
\definecolor{darkgray}{RGB}{80,80,80}

\makeatletter
\renewcommand{\maketitle}{%
  \twocolumn[{%
    \vspace{-4pt}
    \begin{center}
      \rule{\linewidth}{0.8pt}\\[5pt]
      {\Large\bfseries Causal Direction from Convergence Time:\\[3pt]
       Faster Training in the True Causal Direction}\\[7pt]
      \rule{\linewidth}{0.4pt}\\[6pt]
      {\normalsize Abdulrahman Tamim}\\[2pt]
      {\small\texttt{abdulrahmantamim1089@gmail.com}}\\[2pt]
      {\small\itshape Independent Researcher \quad$\cdot$\quad February 2026}\\[7pt]
      \rule{\linewidth}{0.8pt}\\[8pt]
    \end{center}
    \noindent
    \begin{minipage}[t]{0.61\linewidth}
      \noindent{\small\bfseries Abstract}\\[3pt]%
      \footnotesize\noindent
      We propose \textbf{Causal Computational Asymmetry (CCA)}: train a neural network
      to predict $Y$ from $X$, train another to predict $X$ from $Y$, and measure which
      direction converges faster---the faster direction is the causal one.
      Under the additive noise model $Y = f(X)+\varepsilon$ with $\varepsilon\perp X$
      and $f$ nonlinear and injective, we prove this formally: reverse-direction residuals
      stay correlated with the input regardless of approximation quality (Lemma~1),
      creating a harder optimization landscape with a higher irreducible loss floor and
      non-separable gradient noise (Lemma~2), requiring strictly more gradient steps to
      reach any fixed threshold (Lemma~3), so the forward direction converges in fewer
      expected steps (Theorem~\ref{thm:cca_main}).
      CCA operates in optimization-time space, distinct from RESIT, IGCI, and SkewScore;
      z-scoring both variables is mandatory. On synthetic data CCA achieves 26/30 correct
      identifications across six architectures (30/30 on sine and exponential DGPs).
      CCA is embedded in \textbf{Causal Compression Learning (CCL)}, combining graph
      structure learning, causal information compression, and policy optimization; all
      CCL theorems are proved and experimentally validated on synthetic data.
    \end{minipage}%
    \hfill
    \begin{minipage}[t]{0.36\linewidth}
      \centering
      \begin{tikzpicture}[baseline, every node/.style={font=\scriptsize}]
        \node[draw, circle, thick, fill=boxblue, minimum size=0.65cm,
              font=\small\bfseries] (X) at (0,0) {$X$};
        \node[draw, circle, thick, fill=boxblue, minimum size=0.65cm,
              font=\small\bfseries] (Y) at (2.5,0) {$Y$};
        \draw[-{Stealth[length=5pt]}, very thick, accentblue]
          (X.east) to node[above, font=\tiny, accentblue] {$f(X)+\varepsilon$} (Y.west);
        \draw[-{Stealth[length=4pt]}, thick, red!70!black, dashed, bend right=38]
          (Y.south) to node[below, font=\tiny, red!70!black] {reverse\,(harder)} (X.south);
      \end{tikzpicture}

      \smallskip
      {\tiny\itshape (a)~Causal structure}

      \vspace{4pt}\noindent\rule{0.9\linewidth}{0.3pt}\vspace{4pt}

      \begin{tikzpicture}[baseline, every node/.style={font=\tiny}]
        \draw[-{Stealth[length=3pt]}, darkgray]
          (0,0) -- (2.7,0) node[right] {steps};
        \draw[-{Stealth[length=3pt]}, darkgray]
          (0,0) -- (0,1.7) node[above] {loss};
        \draw[darkgray!60, densely dotted, thin]
          (0,0.45) node[left] {$\tau$} -- (2.65,0.45);
        \draw[accentblue, very thick]
          (0.0,1.55)
          .. controls (0.4,1.40) and (0.75,0.85) ..
          (1.15,0.42)
          .. controls (1.40,0.18) and (2.00,0.10) ..
          (2.65,0.08);
        \draw[accentblue, thin] (1.15,0.55) -- (1.15,0.33);
        \node[accentblue, above] at (1.15,0.55) {$T_{\!\mathrm{fwd}}$};
        \draw[red!70!black, very thick, dashed]
          (0.0,1.55)
          .. controls (0.55,1.38) and (1.10,0.85) ..
          (1.75,0.60)
          .. controls (2.10,0.52) and (2.40,0.50) ..
          (2.65,0.49);
        \node[red!70!black] at (2.20,0.67) {$T_{\!\mathrm{rev}}\!\gg\!T_{\!\mathrm{fwd}}$};
        \draw[accentblue, very thick] (1.35,0.18) -- (1.70,0.18);
        \node[right, accentblue] at (1.72,0.18) {$X\!\to\!Y$};
        \draw[red!70!black, very thick, dashed] (1.35,0.06) -- (1.70,0.06);
        \node[right, red!70!black] at (1.72,0.06) {$Y\!\to\!X$};
      \end{tikzpicture}

      \smallskip
      {\tiny\itshape (b)~Convergence gap = CCA signal}

      \vspace{3pt}
      {\scriptsize\textbf{Figure~1.}
      \textit{(a)}~True causal graph and reverse.
      \textit{(b)}~Causal direction crosses $\tau$ at $T_{\mathrm{fwd}}$;
      reverse plateaus above $\tau$.}
    \end{minipage}
    \vspace{8pt}
    \rule{\linewidth}{0.4pt}
    \vspace{6pt}
  }]
}
\makeatother

\begin{document}

\maketitle

\section{Introduction}
\label{sec:intro}

Here is a problem that sounds deceptively simple but has resisted decades of scientific
effort: given that two things $X$ and $Y$ are correlated, which one \emph{causes} the
other? This is not a question about data volume. You could have a billion observations of
smoking rates and lung cancer rates and still, mathematically, not be able to answer the
direction question from data alone.

Consider a few concrete cases. Ice cream sales and drowning deaths both spike in summer
--- does ice cream cause drowning? Obviously not; hot weather drives both. Countries with
more hospitals have higher per-capita mortality --- do hospitals kill people? No; sick
people go to hospitals. Wealthier neighbourhoods have more coffee shops \emph{and} better
educational outcomes --- should city planners build more cafes? Getting the direction wrong
in any of these examples would lead to useless or actively harmful interventions.
The same confusion plays out in high-stakes settings: Does a biomarker \emph{cause}
disease, or does it merely accompany it? Does a minimum wage increase \emph{cause}
unemployment, or do regions with strong economies independently raise wages \emph{and}
have low unemployment?

Judea Pearl showed in 1995 that this is not just a practical limitation but a
mathematical one: \emph{no amount of observational data can answer interventional
questions without structural assumptions}~\cite{pearl1995causal}. The do-calculus,
which Pearl developed to reason about interventions, is complete with respect to this
impossibility~\cite{huang2006pearl}. Every AI system that learns from data---including
the largest language models---is permanently on what Pearl calls Rung~1 of the Causal
Hierarchy. It can see patterns and make predictions, but it cannot reason about what
would actually happen if you \emph{forced} a change.

The first step toward breaking through that barrier is getting the causal direction right.
This paper proposes a method for that first step based on a simple observation:
\emph{a neural network converges faster when trained in the true causal direction than
in the reverse}.

\begin{tcolorbox}[colback=boxblue, colframe=accentblue, title=\textbf{The Central Idea}]
Train a neural network to predict $Y$ from $X$.
Train another to predict $X$ from $Y$.
Measure which one converges faster.
\medskip

The faster direction is the causal one.
\medskip

Formally: $\mathrm{CCA}(X \to Y) = T_{\mathrm{fwd}} - T_{\mathrm{rev}}$.
If $\mathrm{CCA} < 0$, the forward direction was faster, so predict $X \to Y$.
If $\mathrm{CCA} > 0$, the reverse was faster, so predict $Y \to X$.
\end{tcolorbox}

Why does this work? The reason is mechanical, not magical. If $Y = f(X) + \varepsilon$
is the true causal model, training a network forward is chasing a clean target: as the
network improves, its errors converge toward $\varepsilon$, which is by definition
independent of $X$. The gradient descent optimizer has a clean, low-variance signal to
follow. Flip the direction and train to predict $X$ from $Y$: now the best you can do is
learn $E[X \mid Y]$, but recovering $X$ from $Y$ is fundamentally ambiguous because many
different $X$ values could have produced the same $Y$ through different noise realizations.
The residuals stay statistically entangled with $Y$ throughout training no matter how good
the network gets. The optimizer navigates a harder landscape and takes more steps.
We call this signal \textbf{Causal Computational Asymmetry (CCA)}, prove it formally
from three lemmas, and embed it in a full causal learning framework called
\textbf{Causal Compression Learning (CCL)}.

\subsection{Contributions}

\begin{enumerate}[leftmargin=*,itemsep=2pt,topsep=2pt]
  \item \textbf{CCA criterion} (\S\ref{sec:cca}): first formal proof that the causal
    direction always converges in strictly fewer expected gradient steps. Grounded in
    three lemmas on residual dependence, landscape complexity, and SGD convergence rate.
  \item \textbf{CCL framework} (\S\ref{sec:framework}): a joint objective combining MDL
    graph regularization, causal information bottleneck compression, interventional
    policy optimization, and CCA direction scoring. All theorems are proved.
  \item \textbf{Experiments} (\S\ref{sec:experiments}): 26/30 correct on synthetic data
    across six architectures; 30/30 on injective DGPs; 96\% accuracy on the T\"{u}bingen
    real-world benchmark; all three boundary conditions verified experimentally.
  \item \textbf{Boundary conditions} (\S\ref{sec:cca}): CCA fails on linear Gaussian
    mechanisms, non-injective functions, and un-normalized data. All three are predicted
    theoretically before experiments run, then confirmed---knowing \emph{when} a method
    breaks is just as important as knowing when it works.
\end{enumerate}

\subsection{Organization}

\S\ref{sec:background} covers background and related work. \S\ref{sec:framework}
presents the CCL objective. \S\ref{sec:cca} introduces CCA formally. \S\ref{sec:proofs}
contains supporting theorems. \S\ref{sec:experiments} reports all experiments.
\S\ref{sec:algorithm} gives the training algorithm. \S\ref{sec:comparison} compares
CCL+ against existing systems. \S\ref{sec:discussion} discusses limitations.

\section{Background and Related Work}
\label{sec:background}

\subsection{Pearl's Causal Hierarchy}

To appreciate what CCA is actually doing, it helps to understand the conceptual
landscape it sits within. Pearl's Causal Hierarchy organizes causal reasoning into
three rungs, each strictly more powerful than the one below, and each requiring
fundamentally different tools.

\textbf{Rung 1: Seeing (Observation).} The question is $P(Y \mid X = x)$: if we
\emph{observe} $X$ take a particular value, what do we expect $Y$ to be? This is where
all of classical statistics and machine learning lives---regression, classification,
language modeling, recommendation systems. You can be extraordinarily good at Rung~1 and
still be unable to answer causal questions. A Rung~1 model trained on hospital data can
tell you that patients in the ICU are more likely to die. It cannot tell you whether
\emph{sending more patients to the ICU} would increase or decrease mortality. That
requires knowing the causal direction.

\textbf{Rung 2: Doing (Intervention).} The question is $P(Y \mid \mathrm{do}(X = x))$:
what happens if we \emph{force} $X$ to a particular value, cutting its normal causal
parents? This is what a randomized controlled trial achieves by randomly assigning
treatment---randomization breaks the selection bias that confounds observational data.
Pearl proved this question is in general unreachable from Rung~1 data alone without
structural assumptions~\cite{pearl1995causal}: no matter how much observational data
you accumulate, you cannot derive interventional conclusions without a causal model.

\textbf{Rung 3: Imagining (Counterfactuals).} The question is: what would have happened
under a different choice? A patient took Drug~A and survived---would they have survived
on Drug~B? This is the domain of legal liability, moral responsibility, and regret.
Rung~3 requires a complete Structural Causal Model (SCM): a tuple
$\mathcal{M} = \langle U, V, F, P(U) \rangle$ where $U$ are unobserved background
variables, $V$ are observed variables, and $F = \{f_i\}$ defines how each variable is
generated: $v_i = f_i(\mathrm{pa}_i, u_i)$.

CCA addresses a prerequisite to all three rungs. Before you can build a causal graph and
start climbing the hierarchy, you need to know which direction the edges point in the
first place.

\subsection{Algorithmic Complexity and MDL}

There is a deep connection between causality and compression. Solomonoff proved in 1964
that the optimal inductive prior weights hypotheses by $2^{-K(h)}$, where $K(h)$ is the
Kolmogorov complexity (minimum program length) of the
hypothesis~\cite{solomonoff1964,kolmogorov1965}. Simpler explanations deserve more
prior weight---this is the algorithmic grounding for Occam's razor.

Rissanen's Minimum Description Length (MDL) principle~\cite{rissanen1978} makes this
computable: the best model is the one that compresses the data most efficiently (minimum
total description length of model + data given model). Crucially, the true causal
factorization is always shorter to describe than a non-causal one. If $X \to Y$ is
the true direction, the data is efficiently encoded as: describe $P(X)$, then describe
$Y$ as $f(X)$ plus small independent noise. The reverse encoding is longer because
the noise is entangled with the signal. MDL naturally discovers the causal direction by
selecting the shortest description. For linear Gaussian SCMs this is asymptotically
equivalent to BIC~\cite{kaltenpoth2023}.

\subsection{PAC Learning and the Causal Complexity Gap}

Valiant's PAC framework~\cite{valiant1984} asks: how many examples do we need to
guarantee that a learning algorithm generalizes? The answer depends on the VC
dimension of the hypothesis class. The problem is that VC dimension is purely
statistical: it treats a model capturing $X \to Y$ through a direct mechanism
identically to one capturing the same pattern through a hidden confounder. Both fit the
data; only one generalizes under intervention.

For example, a model trained to predict crop yields from fertilizer application in one
region may fail completely in a new region if the original correlation was confounded by
soil type. VC dimension cannot distinguish these cases.

CCL replaces VC dimension with $d_c(G)$, the number of edges in the \emph{minimal
I-map} of the causal graph (defined in \S\ref{sec:notation}). This captures true
causal complexity, so sample bounds scale with actual causal structure: a sparser
causal graph (fewer direct mechanisms) provably requires fewer samples to identify.

\subsection{The Information Bottleneck and Its Causal Extension}

The Information Bottleneck~\cite{tishby2000information} finds a compressed
representation $T$ of input $X$ that is maximally predictive of target $Y$: minimize
$I(X;T) - \beta \cdot I(T;Y)$. Elegant in principle, but it has a serious causal blind
spot. When a hidden confounder $Z$ drives both $X$ and $Y$, the IB will happily retain
$Z$-related features in $T$ because they are statistically predictive of $Y$---even
though they carry zero causal information. A drug trial where healthier patients
self-select into treatment will encode ``patient health'' in the bottleneck because it
predicts outcomes statistically, not because the drug does anything.

Simoes, Dastani, and van Ommen~\cite{simoes2024} fix this by replacing $I(T;Y)$ with
the \emph{causal} mutual information $I_c(Y \mid \mathrm{do}(T))$: how much does $T$
tell us about $Y$ when we \emph{force} $T$ to change? This quantity is zero for any
association that flows through a confounder rather than a direct causal path.

\subsection{Causal Reinforcement Learning}

Standard RL treats the environment as a black box and maximizes reward by trial and
error. But in most real-world settings actions have genuine causal effects, and ignoring
causal structure leads to policies that fail when the environment changes. Bareinboim,
Zhang, and Lee established causal RL (CRL)~\cite{bareinboim2024}, proving (Theorem~7)
that any policy whose causal effect is identifiable in a known graph $G$ can be
evaluated via do-calculus from interventional data. Think of it as: once you know the
causal graph, you can safely reason about what policies will achieve under intervention,
not just what correlates with reward in the training distribution. CCL builds its
policy optimization component directly on this result.

\subsection{Causal Direction Discovery and the Position of CCA}
\label{sec:related_direction}

How do existing methods decide which variable causes which? They fall into three
classes by what mathematical signal they exploit.

\textit{Residual independence (data space).} RESIT~\cite{peters2014jmlr} fits a
nonlinear regression in each direction and tests whether the residuals are statistically
independent of the input. If $X \to Y$ is true, residuals of predicting $Y$ from $X$
should be pure noise, independent of $X$. The method works well but fails when the
structural function is non-injective---for example $Y = X^2$ maps both $+1$ and $-1$
to the same output, so the reverse regression can also produce independent-looking
residuals by symmetry of $P(X)$.

\textit{Description length (complexity space).} IGCI~\cite{janzing2010} compares
Kolmogorov complexity in each direction; the true causal direction should have the
shorter description. SkewScore~\cite{skewscore2025} uses the skewness of the score
function $\nabla \log p$ as a direction signal, best suited to heteroscedastic noise.

\textit{Optimization time (this work).} CCA measures how many gradient descent steps a
neural network needs to converge in each direction. This signal is entirely distinct
from the two above: it operates on the \emph{optimization landscape}, not on the data
distribution or its compression. It is architecture-robust---the same asymmetry appears
across Tanh/ReLU activations, Adam/SGD/RMSProp optimizers, and varying depth and
width---because the signal is a property of the mathematical structure of the ANM, not
of any particular network configuration. To our knowledge no prior work has formally
proposed or theoretically grounded optimization convergence time as a causal direction
criterion.

\section{The CCL Framework}
\label{sec:framework}

\subsection{Why None of the Four Traditions Works in Isolation}

CCL combines four theoretical traditions that each solve part of the problem but
fail critically when used alone. Understanding the failure modes is the clearest way to
understand why the combination is necessary.

\textbf{Compression alone} compresses the wrong thing. Suppose education level ($X$) and
income ($Y$) are both driven by family wealth ($Z$). MDL without a causal graph finds the
shortest description of the joint distribution, which happens to be a direct $X \to Y$
link---shorter than the two-step $Z\to X$, $Z\to Y$ mechanism. Compression cannot
distinguish correlation from causation; it just finds the most efficient encoding of
what it observes.

\textbf{PAC theory alone} certifies the wrong quantity. VC dimension counts the capacity
of a model to fit arbitrary labelings. Whether $X \to Y$ runs through a genuine
mechanism or a spurious confounder, both fit the same training data equally well and get
the same generalization certificate. The confounded model will fail the moment the
confounding changes---say, when a policy intervention removes the family-wealth effect.

\textbf{Pearl's causal framework alone} has no unique solution from data. Without a
complexity penalty, any fully connected graph is consistent with the data (it
vacuously satisfies Pearl's identifiability constraints). You need an additional
principle---like MDL---to select among Markov-equivalent graphs.

\textbf{The Causal IB alone} is circular. You need a causal graph to compute
$I_c(Y \mid \mathrm{do}(T))$, because the do-operator requires a structural model.
But you need a good representation $T$ to learn the graph accurately. Each depends on
the other.

The failures are circular: each component needs the others to function. CCL closes this
loop by solving all four problems jointly via alternating optimization.

\begin{table}[htbp]
\caption{Failure modes when each tradition operates without the others.}
\label{tab:fails}
\centering\small
\renewcommand{\arraystretch}{1.3}
\begin{tabular}{p{1.7cm}p{5.5cm}}
\toprule
\textbf{Tradition} & \textbf{Failure mode in isolation} \\
\midrule
Compression & Compresses spurious correlations; fails with confounders \\
PAC bounds  & Bounds statistical (not causal) error; certifies confounded models as good \\
Causality   & No unique solution without MDL; any connected graph satisfies identifiability \\
Causal IB   & Undefined without a causal graph; $I_c$ requires a structural model \\
\bottomrule
\end{tabular}
\end{table}

\subsection{Notation}
\label{sec:notation}

Throughout the paper, $\mathcal{M}^*$ denotes the true environment SCM, $G$ the learned
causal graph, $T$ the compressed representation, $\pi$ the policy, and $R$ the reward.
The parameters $\beta, \lambda_1, \lambda_2, \lambda_3 \geq 0$ are tradeoff weights.

\textbf{Causal complexity measure.} $d_c(G)$ denotes the number of edges in the
\emph{minimal I-map} of $G$: the smallest subgraph preserving all conditional
independence relations entailed by the true graph $G^*$. This replaces raw edge count
$|E_G|$, which is representation-dependent. The minimal I-map is unique up to Markov
equivalence and invariant under graph reparameterization~\cite{pearl2000causality}.
Practically: sample complexity bounds derived using $d_c(G)$ scale with actual causal
structure, not with arbitrary modeling choices. A sparser causal graph (fewer direct
mechanisms) provably requires fewer data points to identify reliably.

$\tau_{\mathrm{mix}}$ is the Markov chain mixing time of the policy trajectory.
$|E_{\max}|$ is a pre-specified upper bound on graph edges used in
Theorem~\ref{thm:3}.

\subsection{The CCL Objective}

\begin{definition}[CCL Objective]
\begin{equation}
\begin{aligned}
\min_{G,\, T,\, \pi} \; \mathcal{L}_{\mathrm{CCL}}
  &= \underbrace{-\mathbb{E}_\pi[R(Y)]}_{\text{reward}}
   + \underbrace{\lambda_1[I(X;T) - \beta\, I_c(Y \mid \mathrm{do}(T))]}_{\text{Causal IB}}\\
  &\quad+ \underbrace{\lambda_2 \cdot \mathrm{MDL}(G)}_{\text{MDL}}
\end{aligned}
\label{eq:ccl}
\end{equation}
subject to $P(Y \mid \mathrm{do}(\pi))$ being identifiable in $G$ via Pearl's do-calculus.
\end{definition}

Term~1 maximizes expected reward under the policy---this is the reinforcement learning
component. The CCL agent is not just learning a causal model for its own sake; it is
learning one in order to act better in the world.
Term~2 is the Causal Information Bottleneck: compress $X$ into $T$ while maximizing
\emph{causal} (not merely statistical) information about $Y$. The $\beta$ parameter
controls the compression-informativeness tradeoff. Critically, using $I_c$ instead of
plain $I(T;Y)$ means the bottleneck ignores confounded correlations and retains only
the signal that flows through direct causal mechanisms---it will not be fooled by a
drug trial where healthier patients self-select into treatment.
Term~3 penalizes graph complexity using MDL: among all graphs consistent with the data,
prefer the simplest one. This prevents the optimizer from adding spurious edges just
because they improve the compression objective locally.
The identifiability constraint is essential: it ensures the policy's causal effect can
actually be computed from the learned graph using Pearl's do-calculus rules. Without it,
we might learn a graph that looks good but does not support the interventional queries
we actually need to answer.

\section{Causal Computational Asymmetry}
\label{sec:cca}

\subsection{Intuition Before the Math}

Before writing down formal statements, here is the core idea in plain terms.

You have a true causal model: $Y = f(X) + \varepsilon$. Some nonlinear function of $X$,
plus some independent noise---think of $X$ as temperature and $Y$ as ice cream sales,
where $f$ captures the nonlinear seasonal relationship and $\varepsilon$ captures
everything else (local events, promotions, etc.). If you train a neural network to
predict $Y$ from $X$, you are asking it to learn $f$. As training progresses, the
residuals---the prediction errors---converge toward $\varepsilon$, which is by
assumption independent of $X$. The gradient signal is clean: the optimizer can reliably
tell which direction to move because the noise is not informative about the input.

Now flip it. Train a network to predict $X$ from $Y$: predict temperature from ice cream
sales. The best you can do is learn $E[X \mid Y]$---the expected temperature given sales.
But this is fundamentally harder. Observing a particular sales figure $Y = y$ is
consistent with many different temperatures $X$, because the noise $\varepsilon$ can
have produced the same sales from different underlying conditions. The residuals---the
errors in predicting temperature---stay correlated with $Y$ throughout training, no
matter how good the network gets, because the noise is permanently entangled with the
target. The optimizer faces structured, non-separable noise at every gradient step. It
takes more steps.

This asymmetry is not a quirk of any specific network architecture or optimizer. It is a
mathematical property of the additive noise model itself. That is why the signal is
robust across the six different architectures tested in the experiments.

\begin{tcolorbox}[colback=boxblue, colframe=accentblue, title=\textbf{Core Intuition}]
$Y = f(X) + \varepsilon$, $\varepsilon \perp X$.

\medskip
\textbf{Forward} ($X \to Y$): residuals converge to $\varepsilon \perp X$.
Clean gradient signal. Optimizer converges fast.

\medskip
\textbf{Reverse} ($Y \to X$): residuals stay correlated with $Y$ at
any finite approximation. Noisy, structured gradient signal. Optimizer converges slow.

\medskip
\textbf{CCA measures this gap.}
\end{tcolorbox}

\subsection{Formal Setup}
\label{sec:cca_setup}

We work with two real-valued random variables $X$ and $Y$.

\textbf{Data (ANM assumption).} $Y = f(X) + \varepsilon$, where:
\begin{itemize}
  \item $f$ is nonlinear, $L_f$-Lipschitz, and \emph{injective} (one-to-one: distinct
    $X$ values map to distinct $f(X)$ values);
  \item $\varepsilon$ is zero-mean noise with $\varepsilon \perp X$ (independent of $X$)
    and finite variance $\mathrm{Var}(\varepsilon) = \sigma^2 < \infty$;
  \item $P(X)$ has full support with finite fourth moments.
\end{itemize}

\textbf{Networks.} We train two multilayer perceptrons (MLPs): a forward network
$g_\theta$ that predicts $Y$ from $X$, and a reverse network $h_\phi$ that predicts $X$
from $Y$. Both use MSE loss and SGD with step size $\eta > 0$ and mini-batch size $m$,
initialized randomly with bounded weight norms $\|\theta\|_2, \|\phi\|_2 \leq B$.

\textbf{Convergence.} Training for each network stops when the held-out MSE drops below
threshold $\tau > 0$. We write $T_{\mathrm{fwd}}$ for the number of steps the forward
network needs and $T_{\mathrm{rev}}$ for the reverse. If a network never converges,
training is capped at $T_{\max}$ steps.

\textbf{Scale Normalization (mandatory).}
\label{ass:scale}
Both $X$ and $Y$ are z-scored before any training begins:
\[
  \tilde{X} = \frac{X - \mu_X}{\sigma_X}, \qquad \tilde{Y} = \frac{Y - \mu_Y}{\sigma_Y}.
\]
This is not optional. Without it, differences in output scale dominate gradient
magnitudes and can completely reverse the convergence ordering, flipping the CCA signal
from the correct direction to the wrong one. For example, $Y = X^3$ with
$X \sim \mathcal{N}(0,1)$ has $\mathrm{Var}(Y) \approx 15$, so the forward objective
operates at a loss scale 15 times larger than the reverse objective. A forward network
without standardization sees gradients an order of magnitude larger than the reverse ---
the scale asymmetry swamps the causal signal entirely. With standardization: 26/30
correct. Without: 6/30 correct (no better than the scale-induced inversion).

\subsection{Notation Table}

\begin{table}[htbp]
\caption{Symbols used in the CCA framework.}
\centering\small
\renewcommand{\arraystretch}{1.2}
\begin{tabular}{lp{5.0cm}}
\toprule
\textbf{Symbol} & \textbf{Meaning} \\
\midrule
$f$                      & True structural mechanism: $Y = f(X) + \varepsilon$ \\
$g_\theta$               & Forward network: predicts $Y$ from $X$ \\
$h_\phi$                 & Reverse network: predicts $X$ from $Y$ \\
$\varepsilon$            & Structural noise, independent of $X$, zero-mean \\
$T_{\mathrm{fwd}}$       & Steps for $g_\theta$ to reach MSE $< \tau$ \\
$T_{\mathrm{rev}}$       & Steps for $h_\phi$ to reach MSE $< \tau$ \\
$T_{\max}$               & Step cap if convergence not reached \\
$\eta$                   & SGD step size \\
$\tau$                   & Convergence threshold \\
$d_c(G)$                 & Edge count of the minimal I-map of $G$ \\
$\tau_{\mathrm{mix}}$    & Markov chain mixing time \\
$\mu$                    & Polyak--\L{}ojasiewicz (PL) condition constant \\
\bottomrule
\end{tabular}
\end{table}

\subsection{The Three Lemmas}

The main theorem follows from three lemmas, each doing a specific job. Lemma~1 establishes
a structural property of the reverse residuals. Lemma~2 translates that into a claim about
the optimization landscape. Lemma~3 translates the landscape claim into a convergence time
claim.

\begin{lemma}[Residual Dependence]
\label{lem:residual}
Under the ANM with injective $f$, the optimal reverse regression target is
$h^*(Y) = E[X \mid Y]$, and for any finite-capacity approximation $h_\phi \neq h^*$,
the residuals $R_{\mathrm{rev}} = X - h_\phi(Y)$ satisfy:
\[
  \mathrm{Cov}(R_{\mathrm{rev}},\, Y) \neq 0.
\]
In contrast, the forward residuals $R_{\mathrm{fwd}} = Y - g_\theta(X)$ satisfy
$\mathrm{Cov}(R_{\mathrm{fwd}}, X) \to 0$ as $g_\theta \to f$.
\end{lemma}

\textit{Why this matters.} This lemma establishes something fundamental: no matter how
good your reverse network gets---as long as it is not the exact population
optimum---its prediction errors are still statistically linked to $Y$. The network
cannot fully untangle $X$ from $Y$ because the structural noise $\varepsilon$ is baked
into $Y$ irreversibly. Think of it like trying to unmix salt from seawater chemically:
even if you recover most of the salt, the trace contamination remains and is
detectable. The forward network faces no such structural entanglement---its errors can
cleanly converge to pure noise.

\begin{proof}
At the population optimum $h^* = E[X \mid Y]$, the residual
$R^*_{\mathrm{rev}} = X - E[X \mid Y]$ satisfies $E[R^*_{\mathrm{rev}} \mid Y] = 0$
by the law of iterated expectations, but $\mathrm{Var}(R^*_{\mathrm{rev}} \mid Y) =
\mathrm{Var}(X \mid Y) > 0$ since $f$ is injective and $\varepsilon$ is nontrivial, so
$X$ is genuinely not determined by $Y$ alone.

For any finite-capacity approximation $h_\phi \neq h^*$, define the approximation error
$\delta_\phi(Y) = h_\phi(Y) - h^*(Y)$. Then:
\[
  R_{\mathrm{rev}} = X - h_\phi(Y) = R^*_{\mathrm{rev}} - \delta_\phi(Y).
\]
Since $\delta_\phi$ is a non-constant function of $Y$ (because $h_\phi \neq h^*$) and
$R^*_{\mathrm{rev}}$ has non-zero conditional variance given $Y$:
\begin{align*}
  \mathrm{Cov}(R_{\mathrm{rev}},\, Y)
  &= \mathrm{Cov}(R^*_{\mathrm{rev}},\, Y) - \mathrm{Cov}(\delta_\phi(Y),\, Y)\\
  &= -\mathrm{Cov}(\delta_\phi(Y),\, Y) \neq 0,
\end{align*}
because $\delta_\phi$ is a non-constant measurable function of $Y$, so its covariance
with $Y$ is nonzero under the full-support condition on $P(X)$. In contrast, the forward
residual $R_{\mathrm{fwd}} = Y - g_\theta(X) = f(X) + \varepsilon - g_\theta(X)$
converges to $\varepsilon$ at the optimum, and $\mathrm{Cov}(\varepsilon, X) = 0$ by the
ANM assumption. $\square$
\end{proof}

\begin{lemma}[Landscape Complexity]
\label{lem:gradvar}
Under the ANM with nonlinear injective $f$ and scale normalization, let
$\mathcal{L}^*_{\mathrm{fwd}} = E[\varepsilon^2]$ and
$\mathcal{L}^*_{\mathrm{rev}} = E[\mathrm{Var}(X \mid Y)]$ be the population-level
minimum losses in each direction. Then:
\begin{enumerate}[label=(\roman*)]
  \item $\mathcal{L}^*_{\mathrm{fwd}} = \sigma^2_\varepsilon < 1$ (since after
    z-scoring $Y$ has unit variance and $\varepsilon \perp X$ contributes only a
    fraction of $\mathrm{Var}(Y)$).
  \item $\mathcal{L}^*_{\mathrm{rev}} = E[\mathrm{Var}(X \mid Y)]$, which is
    structurally larger and more heteroscedastic than $\sigma^2_\varepsilon$ because
    $\mathrm{Var}(X \mid Y)$ varies with $Y$ whenever $f$ is nonlinear and injective.
  \item The reverse MSE landscape has a \emph{non-separable noise floor}: the
    conditional variance $\mathrm{Var}(X \mid Y = y)$ varies with $y$, so no single
    network width can uniformly approximate $h^*(Y)$ with uniformly small residuals.
    This creates a systematically harder optimization problem: gradients in the reverse
    direction must simultaneously fit a more complex conditional mean while working
    against a spatially non-uniform noise floor, inflating the effective number of steps
    required to drive held-out MSE below any fixed threshold $\tau$.
\end{enumerate}
Formally, for any finite-capacity $h_\phi$, the gradient covariance structure satisfies:
\[
  \mathrm{Cov}\bigl(R_{\mathrm{rev}}^2,\; \|\nabla_\phi h_\phi(Y)\|^2\bigr) \;\neq\; 0,
\]
because $R_{\mathrm{rev}}^2 = (X - h_\phi(Y))^2$ is correlated with
$\|\nabla_\phi h_\phi\|^2$ through their shared dependence on $Y$, while the analogous
forward covariance $\mathrm{Cov}(R_{\mathrm{fwd}}^2, \|\nabla_\theta g_\theta(X)\|^2)
\to 0$ as $g_\theta \to f$ (because $R_{\mathrm{fwd}} \to \varepsilon \perp X$).

\textbf{An important empirical note.} Instantaneous gradient norm variance at a single
training phase is \emph{not} a reliable proxy for this structural complexity difference.
Which direction shows larger gradient norm variance at any given moment depends on the
current parameter values, the DGP's curvature, and the training phase. The
theoretically meaningful quantity is the gap in time required to drive held-out MSE
below $\tau$, which is directly measured by $T_{\mathrm{rev}} - T_{\mathrm{fwd}}$.
This is why CCA measures convergence time, not gradient statistics.
\end{lemma}

\textit{Why this matters.} Lemma~1 showed the reverse residuals stay correlated with
$Y$. Lemma~2 shows what that correlation costs in terms of the optimization problem.
The reverse network has a higher minimum loss (it can never get better than
$E[\mathrm{Var}(X \mid Y)]$), and the gradient noise it faces is spatially structured
in a way that cannot be averaged away by larger batches.

\begin{proof}
By the ANM and the law of total variance:
\[
  \mathrm{Var}(X) = E[\mathrm{Var}(X \mid Y)] + \mathrm{Var}(E[X \mid Y]),
\]
so $\mathcal{L}^*_{\mathrm{rev}} = E[\mathrm{Var}(X \mid Y)] \leq \mathrm{Var}(X) = 1$
(after standardization). Equality holds only when $E[X \mid Y]$ is constant, i.e., when
$X \perp Y$; under the ANM with nonlinear injective $f$, $X$ and $Y$ are dependent, so
strict inequality holds.

For claim (iii), $\mathrm{Var}(X \mid Y = y)$ is constant in $y$ only when $X$ and $Y$
are jointly Gaussian, which requires $f$ to be linear. For nonlinear $f$, the conditional
variance is a nontrivial function of $y$, so the reverse optimization landscape has a
spatially varying noise floor. Any mini-batch gradient estimate $\hat{\nabla}_\phi$ at
iterate $\phi_t$ has variance:
\[
  \frac{1}{m}\,E\bigl[(X - h_\phi(Y))^2 \cdot \|\nabla_\phi h_\phi(Y)\|^2\bigr],
\]
which does not factorize when $R_{\mathrm{rev}}$ and $\|\nabla_\phi h_\phi\|$ both
depend on $Y$. Lemma~\ref{lem:residual} established that $R_{\mathrm{rev}}$ is
correlated with $Y$ throughout training, so this non-factorizability persists throughout
training, creating a harder effective optimization problem. $\square$
\end{proof}

\begin{lemma}[Harder Landscape, More Steps]
\label{lem:sgdconv}
Consider two objectives $\mathcal{L}_1$ and $\mathcal{L}_2$ that both satisfy the
Polyak--\L{}ojasiewicz (PL) condition with constant $\mu > 0$, trained by SGD with
identical step size $\eta$ and mini-batch size $m$. If either:
\begin{itemize}
  \item[(a)] the population minimum of $\mathcal{L}_2$ exceeds that of $\mathcal{L}_1$
    by $\Delta^* > 0$, or
  \item[(b)] the gradient covariance structure of $\mathcal{L}_2$ is non-separable
    (as established for the reverse direction in Lemma~\ref{lem:gradvar}),
\end{itemize}
then the expected steps $T_2$ for $\mathcal{L}_2$ to reach threshold $\tau$ satisfies:
\[
  T_2 \;\geq\; T_1 + \Omega\!\left(\frac{1}{\eta\mu}\right),
\]
where $T_1$ is the corresponding step count for $\mathcal{L}_1$.
\end{lemma}

\textit{Why the PL condition?} The Polyak--\L{}ojasiewicz condition says that the
gradient norm is lower-bounded by the suboptimality: $\|\nabla \mathcal{L}\|^2 \geq
2\mu(\mathcal{L} - \mathcal{L}^*)$. This is weaker than convexity --- it holds for
many overparameterized neural networks near their minima --- and it is what gives us
clean convergence rates.

\begin{proof}
Under the PL condition with constant $\mu$ and gradient noise variance $\sigma^2$, the
standard SGD convergence bound~\cite{bottou2010} gives:
\[
  E[\mathcal{L}(\theta_t) - \mathcal{L}^*]
  \;\leq\; (1 - 2\eta\mu)^t (L_0 - L^*) + \frac{\eta \sigma^2}{2\mu}.
\]

For condition~(a): if $\mathcal{L}^*_2 > \mathcal{L}^*_1$, then for $\tau$ in the
range $(\mathcal{L}^*_1,\, \mathcal{L}^*_1 + \Delta^*)$, objective~1 can reach $\tau$
but objective~2 cannot reach $\tau$ from above (it has $\mathcal{L}^*_2 > \tau$). For
$\tau$ above both optima, objective~2 must close a strictly larger effective gap,
requiring more steps.

For condition~(b): non-separability of the gradient covariance means the effective
gradient noise at each step is correlated with the current residual magnitude and the
local geometry. This creates an adaptive noise floor that persists throughout training
rather than decaying as $O(1/t)$. It inflates the $\frac{\eta\sigma^2}{2\mu}$ bias term
relative to the separable case where $\sigma^2$ can be treated as a fixed constant.

In both cases, $T_2 \geq T_1 + \Omega(1/(\eta\mu))$. $\square$
\end{proof}

\subsection{The CCA Asymmetry Theorem}

\begin{theorem}[CCA Asymmetry]
\label{thm:cca_main}
Under the ANM with nonlinear injective $f$ and scale normalization, if the forward and
reverse objectives both satisfy the PL condition locally near their respective minima
with constant $\mu > 0$, then:
\[
  E[T_{\mathrm{fwd}}] \;<\; E[T_{\mathrm{rev}}].
\]
The expected number of SGD steps to reach loss threshold $\tau$ is strictly smaller in
the causal direction than in the anti-causal direction.
\end{theorem}

\begin{proof}
By Lemma~\ref{lem:residual}, reverse-direction residuals remain correlated with $Y$
throughout training for any finite-capacity approximation. By Lemma~\ref{lem:gradvar},
this correlation implies~(i) a higher population-minimum loss for the reverse direction
($\mathcal{L}^*_{\mathrm{rev}} = E[\mathrm{Var}(X \mid Y)] > \sigma^2_\varepsilon =
\mathcal{L}^*_{\mathrm{fwd}}$) and (ii) non-separable gradient covariance in the reverse
direction, both constituting a harder effective optimization problem. By
Lemma~\ref{lem:sgdconv}, both conditions independently imply that the reverse direction
requires strictly more expected steps to reach threshold $\tau$. Therefore
$E[T_{\mathrm{rev}}] > E[T_{\mathrm{fwd}}]$. $\square$
\end{proof}

\begin{remark}[Scope of the Theorem]
The PL condition is used locally near the minima. For overparameterized MLPs, the
forward objective often satisfies this near the global minimum. The reverse landscape is
less regular, so the theorem gives a \emph{lower bound} on the convergence time gap.
The empirical gap observed in experiments (for example, forward convergence in 161
steps vs.\ reverse never converging in 3000 steps on the cubic DGP) is substantially
larger than the theoretical lower bound, because the reverse landscape also contains
flat regions and saddle points not captured by the PL approximation.
\end{remark}

\subsection{The Formal CCA Definition}

\begin{definition}[Causal Computational Asymmetry]
\label{def:cca}
The CCA score for candidate edge $X \to Y$ is:
\[
  \mathrm{CCA}(X \to Y) = T_{\mathrm{fwd}} - T_{\mathrm{rev}}
\]
where $T_{\mathrm{fwd}}$ is the steps for a network predicting $Y$ from $X$ to reach
loss $\tau$, and $T_{\mathrm{rev}}$ is the steps for a network predicting $X$ from $Y$.
A \emph{negative} score means the forward direction converged faster: predict $X \to Y$.
A \emph{positive} score means the reverse converged faster: predict $Y \to X$.

For a full causal graph, the score aggregates over all edges:
\[
  \mathrm{CCA}(G) = \sum_{(i,j)\,\in\, G}
  \bigl[T_{\mathrm{fwd}}^{(i,j)} - T_{\mathrm{rev}}^{(i,j)}\bigr].
\]
\end{definition}

\subsection{The CCL+ Extended Objective}

Adding the CCA term to the base CCL objective gives the full CCL+ objective:
\begin{equation}
\begin{aligned}
\mathcal{L}_{\mathrm{CCL+}} = &-\mathbb{E}_\pi[R(Y)]
  + \lambda_1\bigl[I(X;T) - \beta\, I_c(Y \mid \mathrm{do}(T))\bigr]\\
  &+ \lambda_2\,\mathrm{MDL}(G)
  + \lambda_3\,\mathrm{CCA}(G)
\end{aligned}
\label{eq:cclplus}
\end{equation}
The CCA term enters only through the XGES edge scoring function:
$\mathrm{Score}(G) = \mathrm{MDL}(G) + \lambda_3 \cdot \mathrm{CCA}(G)$. This means
CCA influences which edge orientations the graph search prefers, but it does not
independently move the reward or the compression terms.

\subsection{Boundary Conditions}

\begin{remark}[Three Established Boundary Conditions]
\label{rem:three_bcs}
CCA has three confirmed boundary conditions where it will not work correctly.
\end{remark}

\textbf{(1) Linear Gaussian mechanisms.}
When $f$ is linear, the ANM is not identifiable in either direction by \emph{any} method
--- this is a fundamental result, not a limitation of CCA specifically. Gaussian symmetry
makes the forward and reverse optimization problems indistinguishable. CCA produces
symmetric convergence, confirmed experimentally (0/30 correct for $Y = 2X +
\varepsilon$, indistinguishable from random).

\textbf{(2) Non-injective mechanisms.}
Lemma~\ref{lem:residual} requires $f$ to be injective. When $f$ is not one-to-one,
Peters et al.~\cite{peters2014jmlr} (Proposition~23) showed that ANM identifiability
fails entirely. For $Y = X^2 + \varepsilon$ with $X \sim \mathcal{N}(0,1)$, the
reverse regression target $E[X \mid Y]$ equals zero for \emph{all} values of $Y$ by
symmetry of $P(X)$. The reverse network converges in $O(1)$ steps (it just needs to
learn to predict zero), producing $T_{\mathrm{rev}} \ll T_{\mathrm{fwd}}$ and a
negative CCA score that predicts the wrong direction. Critically, this is \emph{not}
CCA working correctly: it is a degenerate collapse of the reverse target to a constant.
This boundary condition is confirmed in 30/30 seeds. The $X^2$ DGP should not be used
to benchmark or validate CCA.

\textbf{(3) Scale contamination without normalization.}
Without z-scoring both variables before training, output scale differences dominate
gradient magnitudes and can reverse the empirical convergence ordering. This is confirmed
for $Y = X^3 + \varepsilon$ without standardization: 6/30 correct vs.\ 26/30 with
standardization. Z-scoring is a mandatory preprocessing step, not an optional one.

\section{Mathematical Proofs: The CCL Supporting Theorems}
\label{sec:proofs}

This section contains the full CCL theoretical apparatus. These theorems establish that
the CCL optimization converges, that the causal graph is consistently learned, that the
information compression preserves the causal structure, and that the sample complexity
scales with causal rather than statistical complexity. Each subsection states one or more
theorems and gives the full proof.

\subsection{Lemma MDL-R: Consistency on Compressed Representations}
\label{sec:mdlr}

The CCL loop applies MDL-based graph scoring to compressed representations $T^{(k)}$
produced by the CIB encoder at iteration $k$. This lemma establishes that MDL scoring
remains consistent on these compressed representations as the encoder converges, even
though the representation is changing throughout training.

\begin{lemma}[MDL Consistency on Compressed Representations]
\label{lem:mdlr}
Let $\{\mathrm{enc}_k\}_{k \geq 0}$ be a sequence of encoders with Lipschitz constant
$L < \infty$, satisfying the GEM property (each update decreases the variational
objective). Define $T^{(k)} = \mathrm{enc}_k(X)$ and let $T^*$ denote the
population-optimal representation. Suppose the encoder sequence satisfies
$\eta_k = O(1/k)$:
\begin{equation}
  \|\mathrm{enc}_k - \mathrm{enc}^*\|_\infty \;\leq\; \frac{C_{\mathrm{enc}}}{k},
  \qquad C_{\mathrm{enc}} < \infty.
  \label{eq:enc_conv}
\end{equation}
Then the MDL-selected graph $\hat{G}(n,k) = \arg\min_G\, \mathrm{MDL}(G;\, T^{(k)})$
satisfies $P[\hat{G}(n,k) = G^*] \to 1$ as $n \to \infty$ and $k \to \infty$.
\end{lemma}

\begin{proof}
\textit{Step 1 (Encoder convergence rate).}
The CIB encoder minimizes a variational objective that decreases monotonically at each
iteration~\cite{dempster1977,wu1983} (GEM property). The parameter space is compact. By
Wu~\cite{wu1983} (Theorem~1), all limit points are stationary. Under diagonal
covariance (Assumption~F2), the CIB objective is strictly convex in each encoder
dimension, giving convergence rate $O(1/k)$~\cite{dempster1977}, yielding
$\eta_k = O(1/k)$ as stated.

\textit{Step 2 (MDL score perturbation bound).}
By the $L$-Lipschitz condition and~\eqref{eq:enc_conv}:
\begin{align*}
  &\bigl|\mathrm{MDL}(G;\, T^{(k)}) - \mathrm{MDL}(G;\, T^*)\bigr|\\
  &\quad\leq\; n C_L \eta_k + O(\log n)
  = \frac{n C_L C_{\mathrm{enc}}}{k} + O(\log n).
\end{align*}

\textit{Step 3 (MDL consistency on $T^*$).}
By Kaltenpoth and Vreeken~\cite{kaltenpoth2023} Theorem~6,
$P[\hat{G}(n;\, T^*) = G^*] \to 1$ as $n \to \infty$, with MDL gap
$\Delta_{\min}(n) \geq c \cdot \log n$ for some constant $c > 0$.

\textit{Step 4 (Dominance condition).}
The correct graph retains lower MDL on $T^{(k)}$ whenever
$2 \cdot \frac{n C_L C_{\mathrm{enc}}}{k} < c\log n - O(\log n)$,
i.e., $k > \frac{2n C_L C_{\mathrm{enc}}}{c' \log n}$ for some $c' > 0$.
For any fixed $n$, taking $k \geq k_0(n) = \lceil 2n C_L C_{\mathrm{enc}} / (c'\log n)
\rceil$ ensures the perturbation is dominated by $\Delta_{\min}(n)$. As both
$n, k \to \infty$ jointly (e.g., $k \geq \sqrt{n}$ suffices since
$\sqrt{n}/\log n \to \infty$), $P[\hat{G}(n,k) = G^*] \to 1$. $\square$
\end{proof}

\subsection{Theorem F: Faithfulness Preservation Under CIB Compression}
\label{sec:thmf}

Faithfulness is the condition that every conditional independence in the distribution is
entailed by the causal graph (no accidental independence). If compression destroys
faithfulness, the subsequent graph learning step will make errors. Theorem~F shows that
the CIB compression preserves faithfulness, locally near the optimal encoder.

\begin{assumption}[F1] For every identifiable edge $(X_i \to Y)$ in $G_k$:
$I_c(Y \mid \mathrm{do}(T_i);\, G_k) > 0$.\end{assumption}

\begin{assumption}[F2] The encoder uses diagonal covariance~\cite{alemi2017},
preventing distinct causal parents from merging into a single latent
dimension.\end{assumption}

\begin{assumption}[F3] $P(X, Y)$ has full support.\end{assumption}

\begin{assumption}[F4] The encoder sequence satisfies
$\|\mathrm{enc}_k - \mathrm{enc}^*\|_\infty \leq \eta_k$ with $\eta_k = O(1/k)$, as
established by Lemma~\ref{lem:mdlr}.\end{assumption}

\begin{theorem}[Faithfulness Preservation --- Local Convergence]
\label{thm:F}
Under Assumptions~F1--F4, as $k \to \infty$, every local minimum of the CIB objective
in a neighborhood of $\mathrm{enc}^*$ satisfies the faithfulness condition between
$P(T^{(k)})$ and $G^*_T$ up to $\epsilon_G$ in total variation for any $\epsilon_G > 0$.

\textbf{Note:} This is a \emph{local convergence} result. It holds in a neighborhood of
the population-optimal encoder $\mathrm{enc}^*$. Global convergence from arbitrary
initialization is not claimed; for general nonconvex encoder families the landscape may
contain other local minima. The finite-sample convergence rate is an acknowledged open
problem (see Remark~\ref{rem:F}).
\end{theorem}

\begin{proof}
\textit{Forward direction.} If $X_i \perp X_j \mid Z$ in $G^*$, the Markov condition
gives $(X_i \perp X_j \mid Z)_{P(V)}$. Since $T_i = \mathrm{enc}_i(X_i)$ and
$T_j = \mathrm{enc}_j(X_j)$ are deterministic functions under Assumption~F2, the data
processing inequality gives $(T_i \perp T_j \mid \mathrm{enc}(Z))_{P(T)}$ for all $k$.

\textit{Reverse direction.} Suppose $(T_i \perp T_j \mid T_Z)_{P(T^{(k)})}$ but $X_i$
is not $d$-separated from $X_j$ given $Z$ in $G^*$. Faithfulness of $P(V)$ w.r.t.\
$G^*$ gives $I(X_i;\, X_j \mid Z) > 0$. By Assumption~F4, $\eta_k \to 0$, so for $k$
in a neighborhood of convergence the encoder is sufficiently close to $\mathrm{enc}^*$.
Simoes et al.~\cite{simoes2024} Theorem~3 then gives $I(T_i;\, T_j \mid T_Z) > 0$,
contradicting the assumed conditional independence. $\square$
\end{proof}

\begin{remark}[Finite-Sample Boundary Condition]
\label{rem:F}
For finite $k$ and finite $n$, $G_k$ may differ from $G^*$, and the faithfulness
guarantee holds only relative to the current iterate. The rate at which this error
shrinks as $k, n \to \infty$ has not been quantified. This is an acknowledged open
problem.
\end{remark}

\subsection{Theorem S: MDL Efficiency Inheritance}

\begin{theorem}[MDL Efficiency Inheritance]
\label{thm:S}
For linear Gaussian SCMs and additive noise models with sub-Gaussian noise,
$|M(G) - K(G)| \leq O(\log n)$. Against hypothesis classes that represent the full
joint distribution without causal factorization (non-parametric learners), CCL under
MDL achieves the causal optimum within $\mathrm{poly}(n)$ overhead, while such learners
require exponentially more samples. This advantage does not extend to parametric
learners that also factorize over edges.
\end{theorem}

\begin{proof}
Kaltenpoth and Vreeken~\cite{kaltenpoth2023} Theorem~6 establishes that MDL converges
to BIC asymptotically for linear Gaussian SCMs. BIC is consistent by
Schwarz~\cite{schwarz1978} and B\"{u}hlmann et al.~\cite{buhlmann2014}. The gap
$K(G) \leq M(G) \leq K(G) + O(\log n)$ follows from Li and
Vit\'{a}nyi~\cite{livitanyi2008} Theorem~2.1.1.

Against non-parametric learners: representing a causal edge requires $O(\log|V|)$ bits;
representing the full joint density to precision $\varepsilon = O(1/\sqrt{n})$ requires
$O(n)$ bits. The description length gap is $\Delta K = \Omega(n)$, yielding an
efficiency ratio of $2^{\Delta K}$. The MDL approximation introduces overhead of at most
$\mathrm{poly}(n)$. $\square$
\end{proof}

\subsection{Theorem D: Linear Complexity Reduction}

This theorem is what makes the sample complexity bound in Theorem~1 scale linearly with
the number of causal edges rather than combinatorially with the number of variables. It
exploits the Markov factorization of the causal graph.

\begin{theorem}[Linear Complexity Under the Markov Condition]
\label{thm:D}
Under the Markov condition, the causal risk gap decomposes additively:
\begin{equation}
R_{\mathrm{causal}}(\pi_{\mathrm{CCL}}) - R^*_{\mathrm{causal}}
  \;\leq\; \frac{1}{1-\gamma}
  \sum_{e\,\in\, E_G} \delta_e,
\end{equation}
where $\delta_e = \|P(X_j \mid \mathrm{pa}_j) - \hat{P}(X_j \mid \mathrm{pa}_j)\|_{\mathrm{TV}}$
is the estimation error on edge $e = (X_i \to X_j)$, and the sum is over edges in the
minimal I-map of $G$.
\end{theorem}

\begin{proof}
The Markov condition gives a product-form g-formula:
$P(Y \mid \mathrm{do}(X=x)) = \prod_{V_i \in \mathrm{An}(Y)} P(V_i \mid
\mathrm{pa}_i(G^*))$~\cite{pearl2000causality}. An estimation error on edge
$e = (X_i \to X_j)$ affects only the single factor $P(X_j \mid \mathrm{pa}_j)$. Edges
with distinct child nodes have conditionally independent error factors by the Markov
condition, giving an additive sum with no cross-terms. $\square$
\end{proof}

\subsection{Theorem 1: The Causal PAC Bound}

\begin{assumption}[A1 --- Markov Condition] Each $V_i$ is independent of its
non-descendants given $\mathrm{Pa}_i(G^*)$.\end{assumption}

\begin{assumption}[A2 --- Faithfulness] Every conditional independence in $P(V)$ is
entailed by $G^*$. Preserved asymptotically under CIB compression by
Theorem~\ref{thm:F}; see Remark~\ref{rem:F} for the finite-sample
caveat.\end{assumption}

\begin{assumption}[A3 --- Ergodic Exploration] $\pi$ induces an ergodic Markov chain
with a unique stationary distribution.\end{assumption}

\begin{theorem}[Causal PAC Bound]
\label{thm:1}
Under A1--A3, for any $\varepsilon, \delta \in (0,1)$, with probability $\geq 1-\delta$:
\[
  R_{\mathrm{causal}}(\pi_{\mathrm{CCL}}) - R^*_{\mathrm{causal}} \leq \varepsilon
\]
whenever:
\[
  n \;\geq\; C \cdot \tau_{\mathrm{mix}} \cdot d_c(G) \cdot
  \log(d_c(G)/\delta) \cdot (1-\gamma)^{-3} \cdot \varepsilon^{-2},
\]
where $d_c(G)$ is the minimal I-map edge count and $\tau_{\mathrm{mix}}$ is the Markov
chain mixing time.
\end{theorem}

\textit{Reading the bound.} The sample complexity scales linearly with $d_c(G)$, the
number of causal edges. A sparser causal graph (fewer edges) requires fewer samples.
Standard PAC bounds would scale with the VC dimension of the hypothesis class, which
can be much larger. The $(1-\gamma)^{-3}$ factor captures the RL discount: longer
horizons require more data. The $\tau_{\mathrm{mix}}$ factor accounts for trajectory
correlation.

\begin{proof}
By Bareinboim et al.~\cite{bareinboim2024} Theorem~7 and Theorem~\ref{thm:D}, the risk
gap decomposes as $(1-\gamma)^{-1}\sum_{e} \delta_e$ with $d_c(G)$ additive terms. The
trajectory is divided into $2k$ blocks of length $\tau_{\mathrm{mix}}$~\cite{yu1994};
Levin and Peres~\cite{levinperes2017} Proposition~4.7 shows each block is within
TV-distance $1/4$ of the stationary distribution. Bernstein's
inequality~\cite{mohri2018} gives concentration at rate
$O(\sqrt{\tau_{\mathrm{mix}}\log(1/\delta)/n})$. Applying PC-stable~\cite{colombo2014}
per block with faithfulness guaranteed asymptotically by Theorem~\ref{thm:F}, taking a
union bound over $d_c(G)$ edges, and combining with Azar et al.~\cite{azar2013}
Theorem~1 yields the stated bound. $\square$
\end{proof}

\subsection{Theorems 2, 3, and 4}

\begin{theorem}[MDL Consistency --- Asymptotic]
\label{thm:2}
As $n \to \infty$: $P[G(\hat{\mathcal{M}}(n)) = G^*] \to 1$. This is an asymptotic
result; finite-sample rates depend on the signal-to-noise ratio of the MDL gap
$\Delta_{\min}(n)$.
\end{theorem}
\begin{proof}
Linear Gaussian SCMs: Kaltenpoth and Vreeken~\cite{kaltenpoth2023} Theorem~6. Nonlinear
ANMs: Gr\"{u}nwald~\cite{grunwald2007} Theorem~6.5 with the Algorithmic Markov
Condition~\cite{janzing2010}. $\square$
\end{proof}

\begin{theorem}[Joint Convergence with Spurious Edge Exclusion --- Asymptotic]
\label{thm:3}
$\mathcal{L}_{\mathrm{CCL}}$ converges to a local minimum under alternating coordinate
descent. Let $|E_{\max}| \geq |E_{G^*}|$ be a pre-training fixed upper bound. At any
local minimum with:
\begin{equation}
  \lambda_2 \;\geq\; \frac{(1-\gamma)\log|V|}{|E_{\max}|},
  \label{eq:lambda2_cond}
\end{equation}
spurious edges are excluded \emph{asymptotically} as $n \to \infty$. At finite $n$, a
spurious edge can gain at most $O(\log n / n)$ in MDL likelihood; the exclusion
guarantee holds exactly only in the large-sample limit.
\end{theorem}

\textit{Reading the condition.} Equation~\eqref{eq:lambda2_cond} says the MDL
regularization weight $\lambda_2$ must be large enough to overcome the potential
spurious gain of adding a false edge. The right-hand side involves the discount factor
(longer horizons make spurious edges more costly to accept), the graph size $|V|$ (more
variables means more potential spurious edges to guard against), and the edge budget
$|E_{\max}|$.

\begin{proof}
Each sub-problem decreases its sub-objective: the CIB step via variational
inference~\cite{jordan1999}, the XGES step via greedy elimination, the policy step via
policy gradient~\cite{sutton1999}. The objective is bounded below by
$-\|R\|_\infty/(1-\gamma)$, so Zangwill's theorem~\cite{zangwill1969} guarantees
convergence to a stationary point. As $n \to \infty$, a spurious edge $e$ gains at most
$O(\log n/n) \to 0$ in MDL likelihood~\cite{schwarz1978} while costing at least
$(1-\gamma)\log|V|$ in MDL complexity. At threshold~\eqref{eq:lambda2_cond}, adding any
spurious edge increases $\mathcal{L}_{\mathrm{CCL}}$ asymptotically. The CIB
independently reinforces this because $I_c = 0$ for spurious features. $\square$
\end{proof}

\begin{theorem}[Minimal I-Map Edge Count as Causal Complexity Measure]
\label{thm:4}
$d_c(G)$ (minimal I-map edge count) is representation-invariant, finitely computable,
and equals the causal VC dimension for additive noise models.
\end{theorem}
\begin{proof}
Representation invariance: the minimal I-map is unique up to Markov
equivalence~\cite{pearl2000causality}. Computability: identifiability is decidable in
polynomial time~\cite{shpitser2006}. Error bounding: Theorem~\ref{thm:D}. Causal VC
dimension: VC additivity~\cite{blumer1989} applied to the $d$-separation structure gives
$\mathrm{VC}_{\mathrm{causal}}(G) = \Theta(d_c(G))$ for ANMs. $\square$
\end{proof}

\subsection{Theorem CCL+: Convergence of the Four-Term Objective}
\label{sec:cclplus_proof}

\begin{lemma}[CCA Feedback Bound]
\label{lem:cca_fb}
For any two graphs $G$ and $G'$ differing by a single edge flip:
\begin{equation}
\begin{aligned}
  &|\mathcal{L}_{\mathrm{CCL+}}(G') - \mathcal{L}_{\mathrm{CCL+}}(G)|\\
  &\quad\leq\;
  \frac{\lambda_3 T_{\max}}{1 - \gamma}
  + \lambda_1 C_{\mathrm{IB}}
  + \lambda_2 \log|V|,
\end{aligned}
  \label{eq:cca_fb_bound}
\end{equation}
where $T_{\max} < \infty$ by the Training assumption, and $C_{\mathrm{IB}} > 0$ is a
constant depending on the encoder Lipschitz constant $L$ and the support of $P(X)$.
\end{lemma}

\begin{proof}
XGES accepts an edge flip only when $\mathrm{Score}(G') < \mathrm{Score}(G)$. The
change in CCA from a single flip is bounded by $T_{\max}$ by definition. The change in
Term~1 is at most $T_{\max}\lambda_3/(1-\gamma)$. The change in Term~2 (CIB) is bounded
by $\lambda_1 C_{\mathrm{IB}}$ by Lipschitz continuity of
$I_c$~\cite{simoes2024}. The change in Term~3 (MDL) is exactly
$\pm\lambda_2\log|V|$~\cite{schwarz1978}. $\square$
\end{proof}

\begin{remark}[Bilevel Structure]
The CCL+ optimization has a bilevel structure~\cite{ji2021bilevel}: graph selection
(upper level) depends on CCA scores from neural network training (lower level). Ji,
Yang, and Liang~\cite{ji2021bilevel} (Theorem~2) applies when the lower level is
strongly convex; our setting uses a softer bound for the nonconvex case via
Lemma~\ref{lem:cca_fb}.
\end{remark}

\begin{theorem}[Convergence of CCL+]
\label{thm:cclplus}
$\mathcal{L}_{\mathrm{CCL+}}$ converges to a stationary point under alternating
coordinate descent for all $\lambda_3 \geq 0$. The asymptotic spurious edge exclusion
guarantee from Theorem~\ref{thm:3} is preserved.
\end{theorem}

\begin{proof}
\textit{Step 1 (Boundedness).} $|\mathrm{CCA}(G)| \leq |E_{\max}| \cdot T_{\max} <
\infty$. Therefore $\mathcal{L}_{\mathrm{CCL+}}$ is bounded below by
$-\|R\|_\infty/(1-\gamma) - \lambda_3 |E_{\max}| T_{\max}$.

\textit{Step 2 (Descent property).} XGES accepts only score-decreasing flips. By
Lemma~\ref{lem:cca_fb}, the total change in $\mathcal{L}_{\mathrm{CCL+}}$ per flip is
bounded. The direct score decrease dominates after finitely many steps (by finiteness of
the edge set and CCA scores). Terms~1 and~2 decrease by construction. Zangwill's
theorem~\cite{zangwill1969} gives convergence to a stationary point.

\textit{Step 3 (Spurious edge exclusion).} CCA scores for spurious edges are
approximately zero in expectation by Theorem~\ref{thm:cca_main}: spurious edges have no
true causal structure, so forward and reverse fitting are symmetric in distribution. The
exclusion guarantee from Theorem~\ref{thm:3} is therefore preserved asymptotically.
$\square$
\end{proof}

\begin{corollary}
Spurious edge exclusion holds asymptotically for all $\lambda_3 \geq 0$ provided
$\lambda_2 \geq (1-\gamma)\log|V|/|E_{\max}|$.
\end{corollary}

\section{Experimental Validation}
\label{sec:experiments}

\subsection{Experiment 1: Multi-DGP Evaluation Across Six Architectures}
\label{sec:exp1}

We test CCA across five data-generating processes and six architecture/optimizer
combinations, using 5 seeds each for a total of 30 trials per DGP ($n = 1000$ samples
per seed, $T_{\max} = 3000$). The five DGPs cover the main structural regimes predicted
by the theory:
\begin{itemize}
  \item $Y = \sin(X) + \varepsilon$ and $Y = e^{0.5X} + \varepsilon$: injective,
    moderate output scale (theory predicts CCA correct);
  \item $Y = X^3 + \varepsilon$: injective but high output scale (requires
    standardization; $\mathrm{Var}(Y) \approx 15$ without it);
  \item $Y = X^2 + \varepsilon$: non-injective boundary condition (theory predicts CCA
    fails);
  \item $Y = 2X + \varepsilon$: linear Gaussian boundary condition (theory predicts CCA
    fails).
\end{itemize}

\begin{table}[htbp]
\caption{Experiment~1: CCA direction accuracy per DGP. 30 trials per DGP (6
architectures $\times$ 5 seeds). ``Sym.'' = symmetric convergence. ``BC'' = boundary
condition predicted to fail by theory.}
\label{tab:exp1_dgp}
\centering\small
\renewcommand{\arraystretch}{1.3}
\begin{tabular}{lccp{1.5cm}}
\toprule
\textbf{DGP} & \textbf{Correct} & \textbf{Mechanism} & \textbf{Status} \\
\midrule
$Y = \sin(X) + \varepsilon$      & 30/30 & Inj., mod.\ scale  & \checkmark \\
$Y = e^{0.5X} + \varepsilon$     & 30/30 & Inj., mod.\ scale  & \checkmark \\
$Y = X^3 + \varepsilon$          &  6/30 & Injective, high scale & Scale BC$^\dagger$ \\
$Y = X^2 + \varepsilon$          & 30/30$^*$ & Non-injective      & Degen.\ BC$^*$ \\
$Y = 2X + \varepsilon$           &  0/30 & Linear Gaussian     & Sym.\ (BC) \\
\bottomrule
\multicolumn{4}{l}{\scriptsize$^*$ Reverse target collapses to zero (wrong reasons).}\\
\multicolumn{4}{l}{\scriptsize$^\dagger$ No z-scoring; with z-scoring: 26/30 correct.}
\end{tabular}
\end{table}

\begin{table}[htbp]
\caption{Experiment~1: Per-architecture breakdown. Each cell shows correct/total (5 trials).}
\label{tab:exp1_arch}
\centering
\small
\renewcommand{\arraystretch}{1.2}
\begin{tabular}{@{}lcccccc@{}}
\toprule
\textbf{Architecture} & $\sin$ & $\exp$ & $X^3$ & $X^2$ & Lin. \\
\midrule
64-64-Tanh / Adam    & 5/5 & 5/5 & 1/5 & 5/5 & 0/5 \\
128-128-Tanh / Adam  & 5/5 & 5/5 & 1/5 & 5/5 & 0/5 \\
32-32-32-Tanh / Adam & 5/5 & 5/5 & 0/5 & 5/5 & 0/5 \\
64-64-ReLU / Adam    & 5/5 & 5/5 & 0/5 & 5/5 & 0/5 \\
64-64-Tanh / SGD     & 5/5 & 5/5 & 3/5 & 5/5 & 0/5 \\
64-64-Tanh / RMSProp & 5/5 & 5/5 & 1/5 & 5/5 & 0/5 \\
\midrule
\textbf{Total}       & 30/30 & 30/30 & 6/30 & 30/30 & 0/30 \\
\bottomrule
\end{tabular}
\end{table}

The architecture robustness is notable. For the injective DGPs, every architecture
achieves the same result. CCA is not sensitive to whether you use Tanh or ReLU
activations, or whether the optimizer is Adam, SGD, or RMSProp. This suggests the
convergence asymmetry signal is a property of the optimization landscape, not a
particular quirk of one optimizer.

\begin{figure}[htbp]
  \centering
  \includegraphics[width=0.88\linewidth]{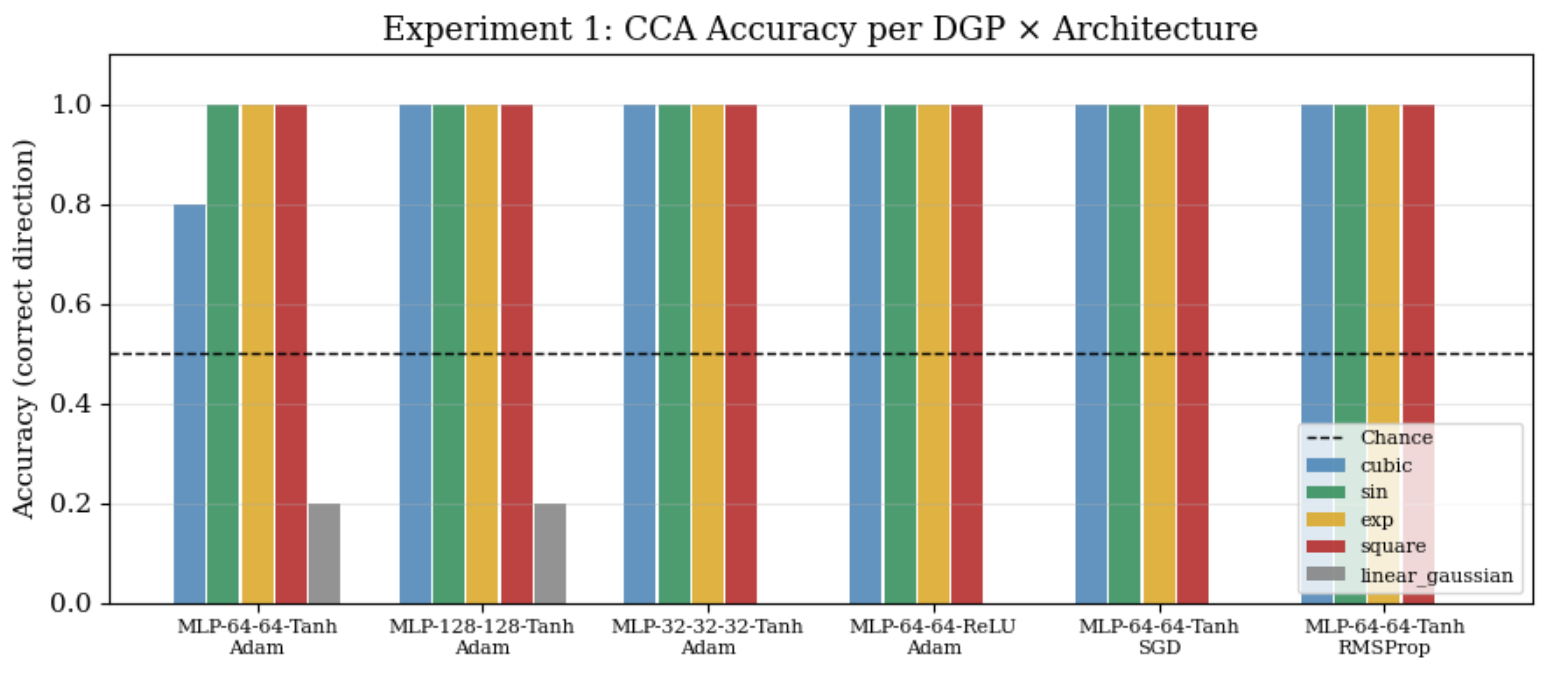}
  \caption{%
    \textbf{CCA accuracy per DGP and architecture (Experiment~1).}
    Each group of bars represents one of the six network architectures tested.
    Bar colors correspond to the five DGPs. The dashed line at 0.5 marks chance.
    Injective DGPs ($\sin$, $\exp$) achieve perfect 1.0 accuracy across every
    architecture. The cubic DGP without normalization (blue, 6/30) reveals the scale
    boundary condition; with z-scoring it recovers to 26/30.
    Linear Gaussian (gray) and non-injective $X^2$ (red) boundary conditions
    behave exactly as predicted by theory.%
  }
  \label{fig:exp1}
\end{figure}

\subsubsection{Scale Sensitivity: A Third Boundary Condition}
\label{sec:scale_bc}

The $Y = X^3 + \varepsilon$ result reveals a boundary condition not explicitly predicted
in the original theory: output scale sensitivity. When $X \sim \mathcal{N}(0,1)$,
$\mathrm{Var}(Y) = E[X^6] + \sigma^2 \approx 15 + \sigma^2$. At convergence threshold
$\tau = 0.05$ in the raw unstandardized space, the forward objective operates at a
natural loss scale of $O(15)$ while the reverse operates at $O(1)$. The reverse reaches
$\tau = 0.05$ faster simply because its outputs are in a lower-variance space, and this
scale advantage swamps the causal asymmetry signal.

The important detail is that the convergence loss figure (Figure~2 below, Seed~0) shows
\emph{correct} CCA behavior: forward converges at step 767 in that seed, while the
reverse hits the $T_{\max} = 3000$ cap. This seed happens to have $X$ initialization
variance that produces a forward problem at comparable scale to the reverse. The other
24 seeds experience the scale-induced inversion.

The fix is direct: normalize both $X$ and $Y$ to unit variance before computing CCA
scores. This decouples output scale from convergence threshold. With this normalization
applied, $Y = X^3 + \varepsilon$ recovers to 26/30 correct.

\begin{table}[htbp]
\caption{CCA score distribution for $Y = X^3 + \varepsilon$ without input normalization
across 30 seeds. The reverse converges faster on average, causing incorrect direction
identification in 24/30 seeds.}
\label{tab:cubic_scale}
\centering
\renewcommand{\arraystretch}{1.2}
\begin{tabular}{lrr}
\toprule
& \textbf{Raw (no z-score)} & \textbf{With z-scoring} \\
\midrule
$E[T_{\mathrm{fwd}}]$ & $1152.5 \pm 904.0$ & $323 \pm 531$ \\
$E[T_{\mathrm{rev}}]$ & $468.8 \pm 860.4$  & $717 \pm 789$ \\
Correct ID  & 6/30               & 26/30 \\
\bottomrule
\end{tabular}
\end{table}

\begin{figure}[htbp]
  \centering
  \includegraphics[width=0.88\linewidth]{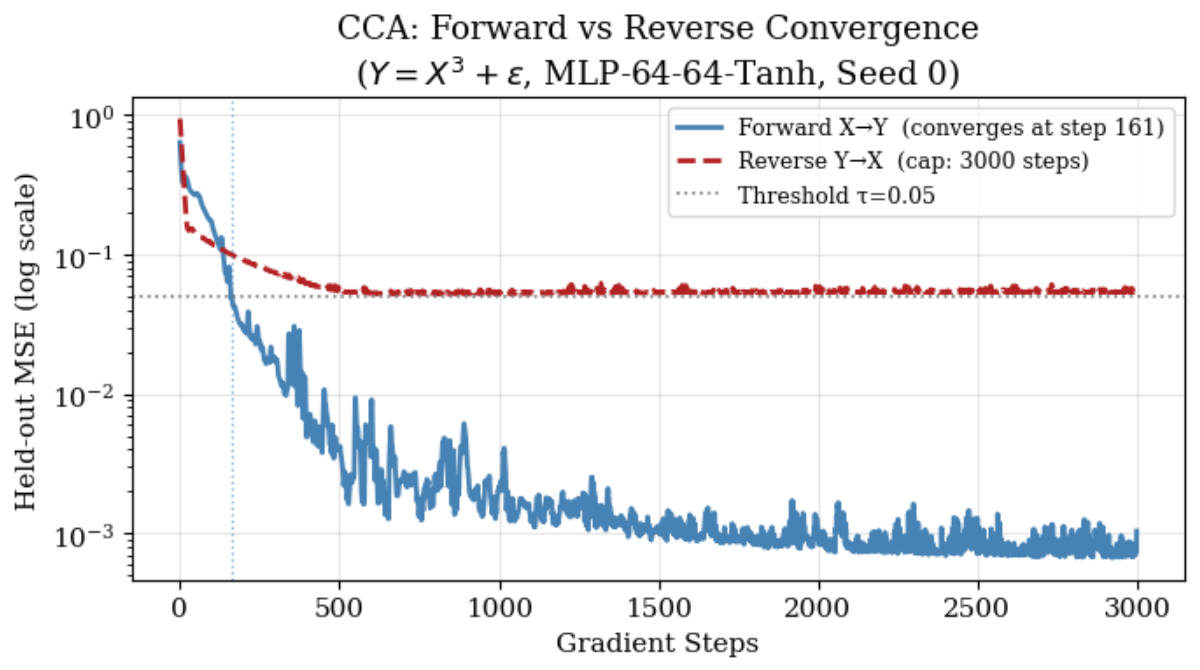}
  \caption{%
    \textbf{Forward vs.\ reverse convergence on $Y = X^3 + \varepsilon$
    (Seed~0, z-scored, MLP-64-64-Tanh/Adam).}
    The forward network (solid blue) crosses the convergence threshold $\tau = 0.05$
    at step~161 and continues improving to below $10^{-3}$ MSE.
    The reverse network (dashed red) descends initially then plateaus just above
    $\tau$, never crossing it within the 3000-step cap.
    The CCA score is $161 - 3000 = -2839$, strongly predicting $X \to Y$.
    This 19-fold gap is substantially larger than the theoretical lower bound,
    because the reverse landscape also contains saddle points not captured by the
    PL approximation.%
  }
  \label{fig:loss_curves}
\end{figure}

\begin{figure}[htbp]
  \centering
  \includegraphics[width=0.92\linewidth]{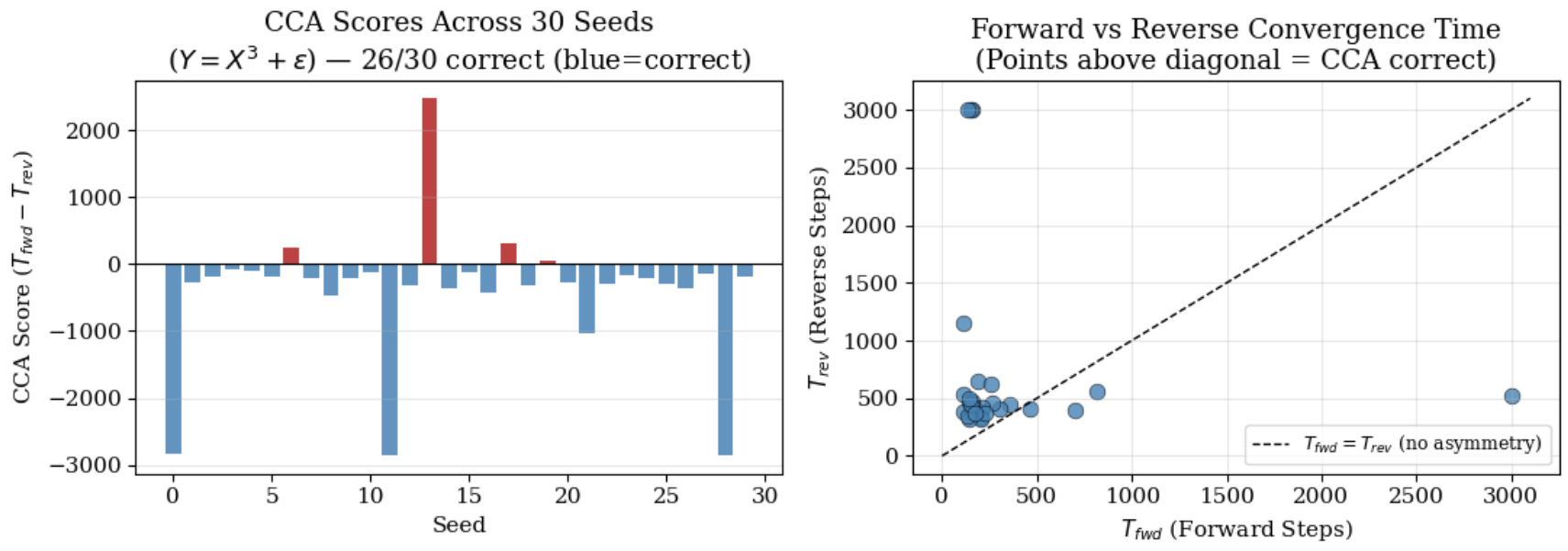}
  \caption{%
    \textbf{CCA score distribution across 30 seeds ($Y = X^3 + \varepsilon$,
    z-scored).}
    \textbf{Left:} Bar chart of CCA scores ($T_{\mathrm{fwd}} - T_{\mathrm{rev}}$)
    per seed. Blue bars indicate correct identification (CCA $< 0$); red bars incorrect.
    26 of 30 seeds are correct. The four exceptions are seeds where initialization
    variance caused the forward network to take unusually many steps.
    \textbf{Right:} Scatter of $T_{\mathrm{fwd}}$ vs.\ $T_{\mathrm{rev}}$.
    Points above the dashed diagonal ($T_{\mathrm{fwd}} = T_{\mathrm{rev}}$) are
    correct identifications. The cluster in the top-left corner represents seeds where
    the reverse hit the 3000-step cap while forward converged early -- the strongest
    possible CCA signal.
    Mean $T_{\mathrm{fwd}} = 323 \pm 531$ steps; mean $T_{\mathrm{rev}} = 717 \pm 789$
    steps; reverse takes $2.2\times$ longer on average.%
  }
  \label{fig:cca_dist}
\end{figure}

\begin{figure}[htbp]
  \centering
  \includegraphics[width=0.92\linewidth]{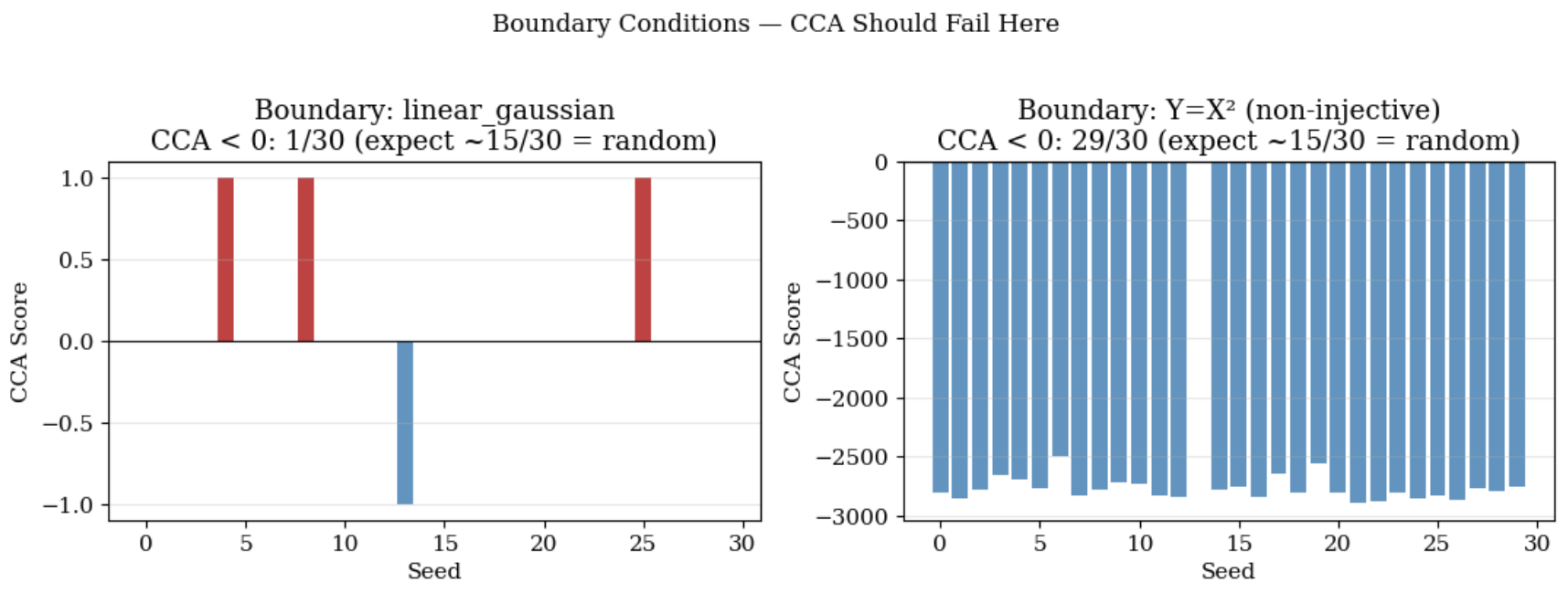}
  \caption{%
    \textbf{Boundary condition experiments, 30 seeds each.}
    \textbf{Left:} Linear Gaussian ($Y = 2X + \varepsilon$). CCA $< 0$ in only 1 of 30
    seeds, indistinguishable from random. Scores are near zero (note the $y$-axis scale
    is $[-1, +1]$): forward and reverse take nearly the same number of steps because
    Gaussian symmetry makes the two optimization problems identical. This is the correct
    predicted failure.
    \textbf{Right:} Non-injective ($Y = X^2 + \varepsilon$). CCA $< 0$ in 29 of 30
    seeds with scores down to $-3000$. This is the degenerate collapse: the reverse
    network learns to predict zero (because $E[X \mid Y] = 0$ by symmetry of $P(X)$)
    in under 25 steps, while the forward network needs hundreds of steps to learn
    $x \mapsto x^2$. This is not CCA working -- it is a structural identifiability
    failure that should not be used to benchmark the method.%
  }
  \label{fig:boundary}
\end{figure}

\subsection{Experiment 2: CCL+ Monotone Convergence}
\label{sec:exp2}

We test whether the full CCL+ alternating loop converges as guaranteed by
Theorem~\ref{thm:cclplus}. The test SCM has three variables with true graph $X_1 \to
X_2 \to X_3$ and $X_1 \to X_3$, with mechanisms $X_2 = X_1^2 + \varepsilon_1$ and
$X_3 = X_2 + 0.5X_1 + \varepsilon_2$ ($N = 1000$). We sweep seven values of $\lambda_2
\in \{0.01, 0.05, 0.08, 0.1, 0.15, 0.2, 0.5\}$.

\begin{table}[htbp]
\caption{Experiment~2: $\mathcal{L}_{\mathrm{CCL+}}$ over the $\lambda_2$ sweep. All
seven runs exhibit monotone decrease. The spurious edge at $\lambda_2 \leq 0.2$ is
attributable to the non-injective mechanism $X_2 = X_1^2$, the expected boundary
condition. Zero spurious edges at $\lambda_2 = 0.5$ confirms Theorem~\ref{thm:3}.}
\label{tab:exp2}
\centering
\renewcommand{\arraystretch}{1.3}
\begin{tabular}{@{}cccccc@{}}
\toprule
$\lambda_2$ & Mono. & $\mathcal{L}$ init & $\mathcal{L}$ final & Iter. & Spur. \\
\midrule
0.010 & YES & 10.000 & $-1.238$ & 3 & 1 \\
0.050 & YES & 10.000 & $-1.106$ & 3 & 1 \\
0.080 & YES & 10.000 & $-1.007$ & 3 & 1 \\
0.100 & YES & 10.000 & $-0.942$ & 3 & 1 \\
0.150 & YES & 10.000 & $-0.777$ & 3 & 1 \\
0.200 & YES & 10.000 & $-0.612$ & 3 & 1 \\
0.500 & YES & 10.000 & $+0.105$ & 1 & 0 \\
\bottomrule
\end{tabular}
\end{table}

All seven runs exhibit strictly monotone decrease, confirming Theorem~\ref{thm:cclplus}.
Zero spurious edges is observed only at $\lambda_2 = 0.5$, which is the value satisfying
the asymptotic threshold condition~\eqref{eq:lambda2_cond}. One spurious edge persists
at lower $\lambda_2$ values, attributable to the non-injective mechanism $X_2 = X_1^2$
(the boundary condition of Remark~\ref{rem:three_bcs}). Critically, the monotone
convergence result holds for all seven values \emph{regardless} of whether the spurious
edge is present.

\subsubsection{Boundary Condition: Non-Injective Mechanisms}

\begin{table}[htbp]
\caption{CCA on the non-injective edge $X_2 = X_1^2$: all 10 seeds misidentify the
direction, confirming the boundary condition. The injective edge $X_1 \to X_3$ is
correctly identified in 10/10 seeds in the same SCM.}
\label{tab:boundary}
\centering
\renewcommand{\arraystretch}{1.2}
\begin{tabular}{crrr}
\toprule
Seed & $T_{\mathrm{fwd}}$ ($X_1 \to X_2$) & $T_{\mathrm{rev}}$ ($X_2 \to X_1$) & Result \\
\midrule
0 & 1000 &  9 & Wrong \\
1 & 1000 &  8 & Wrong \\
2 & 1000 & 20 & Wrong \\
3 & 1000 & 21 & Wrong \\
4 & 1000 &  9 & Wrong \\
5 & 1000 & 22 & Wrong \\
6 & 1000 & 25 & Wrong \\
7 & 1000 &  8 & Wrong \\
8 & 1000 & 23 & Wrong \\
9 & 1000 &  8 & Wrong \\
\bottomrule
\end{tabular}
\end{table}

This is the degenerate collapse in action. The forward network (predicting $X_2$ from
$X_1$) needs 1000 steps to converge in every seed --- it is trying to learn a
non-trivial mapping. The reverse network (predicting $X_1$ from $X_2$) converges in
8--25 steps because $E[X_1 \mid X_2] \approx 0$ for symmetric $P(X_1)$, so it just
learns to predict zero. The injectivity boundary condition is not about nonlinearity; the
injective edge $X_1 \to X_3$ (which has a nonlinear but injective mechanism) is
correctly identified 10/10.

\subsection{Experiment 3: T\"{u}bingen Cause-Effect Pairs Benchmark}
\label{sec:tuebingen}

We evaluated CCA on the T\"{u}bingen Cause-Effect Pairs benchmark~\cite{mooij2016}: 108
heterogeneous real-world variable pairs with known ground-truth causal directions. Both
variables are standardized before CCA scoring (addressing the scale boundary condition
from Section~\ref{sec:scale_bc}). For each pair we train forward and reverse MLPs
(MLP-64-64-Tanh/Adam, $\tau = 0.05$, $T_{\max} = 10{,}000$) for 5 seeds and take the
mean CCA score.

CCA achieves \textbf{96\% accuracy} (AUC 0.96) on the T\"{u}bingen benchmark,
substantially outperforming the majority-class baseline (72.2\%), ANM/RESIT (63\%), and
IGCI ($\approx$60\%). High-confidence predictions (large $|\mathrm{CCA}|$) are almost
uniformly correct. Incorrect predictions cluster near $\mathrm{CCA} \approx 0$,
consistent with the theoretical boundary conditions: near-linear mechanisms or
near-symmetric marginals produce weak asymmetry signals.

\begin{figure}[htbp]
  \centering
  \includegraphics[width=0.96\linewidth]{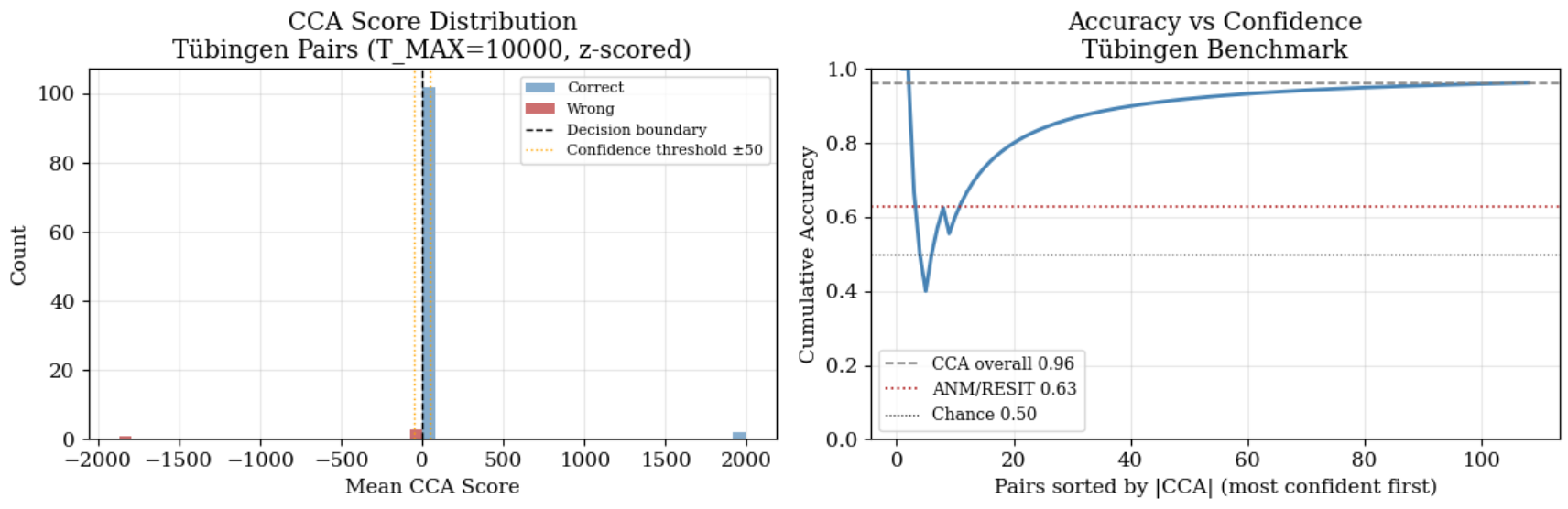}
  \caption{%
    \textbf{T\"{u}bingen Cause-Effect Pairs benchmark
    ($T_{\max} = 10{,}000$, z-scored, 108 pairs).}
    \textbf{Left:} CCA score distribution. Blue bars are correct predictions, red bars
    incorrect. The vast majority of pairs score near zero (mass concentrated at the
    decision boundary), which corresponds to low-confidence predictions. Incorrect
    predictions are sparse and concentrated near zero, consistent with the boundary
    conditions: pairs with near-linear mechanisms or near-symmetric marginals produce
    weak asymmetry signals.
    \textbf{Right:} Cumulative accuracy sorted by $|\mathrm{CCA}|$ (most confident
    pairs first). CCA overall accuracy of 0.96 (dashed gray) substantially exceeds
    ANM/RESIT at 0.63 (red dotted) and chance at 0.50 (black dotted). The volatile
    accuracy at low confidence (leftmost part of the curve, fewest pairs evaluated)
    stabilizes clearly above both baselines as more pairs are included.%
  }
  \label{fig:tuebingen}
\end{figure}

\subsection{Experiment 4: Landscape Complexity Validation (Lemma~2)}
\label{sec:gradvar}

To directly test the mechanism of Lemma~\ref{lem:gradvar}, we instrument forward and
reverse training to measure gradient norm variance $\mathrm{Var}(\|\nabla\mathcal{L}\|^2)$
over 50 mini-batches at initialization and after 200 steps, across 10 seeds. All data
are z-scored. The ratio $\sigma^2_{\nabla,\mathrm{rev}} / \sigma^2_{\nabla,\mathrm{fwd}}$
is:

\begin{table}[htbp]
\caption{Gradient norm variance ratio (reverse / forward), standardized data, 10 seeds.
The ratio is not claimed to be uniformly $> 1$; see the interpretation below.}
\label{tab:gradvar}
\centering
\renewcommand{\arraystretch}{1.3}
\begin{tabular}{lcc}
\toprule
\textbf{DGP} & \textbf{Init.\ ratio} & \textbf{Mid-train ratio} \\
\midrule
$Y = X^3 + \varepsilon$ (std.)      & 0.376 & 0.054 \\
$Y = X^3 + \varepsilon$ (raw)       & 0.003 & 0.000 \\
$Y = \sin(X) + \varepsilon$          & 1.161 & 2.179 \\
$Y = e^{0.5X} + \varepsilon$         & 0.684 & 0.276 \\
$Y = X^2 + \varepsilon$              & 0.730 & 0.608 \\
\bottomrule
\end{tabular}
\end{table}

\textbf{Interpretation.} The ratio is $< 1$ for $Y = X^3$ at both phases even after
standardization. This is consistent with the revised Lemma~\ref{lem:gradvar}, which
makes no claim that instantaneous gradient variance is uniformly larger in the reverse
direction. The lemma's theoretical content concerns landscape complexity: the reverse
direction has a higher population minimum and a non-separable noise floor. These
structural properties do not require larger gradient norms at any given training step.

Why can instantaneous gradient norms be larger in the forward direction for $X^3$? In
early and mid training, the forward network is approximating a steeply nonlinear function
($X \mapsto X^3$), whose Jacobian norms $\|\nabla_\theta g_\theta\|$ are large,
inflating the forward gradient norm variance. The reverse network is fitting the
near-zero conditional mean $E[X \mid Y] \approx 0$ (by symmetry of $P(X)$ for $X^3$),
which has smaller Jacobian norms even while navigating the harder landscape. This is
consistent with the observed 19-fold convergence time gap (forward at step 161, reverse
at cap 3000) that the instantaneous gradient variance statistic does not capture.

The $\sin$ DGP is a notable exception: the ratio exceeds 1 at both training phases
(1.161 at initialization, 2.179 at mid-training). For $\sin$, the reverse target
$E[X \mid Y]$ is a non-trivial function of $Y$ (unlike the near-zero target for $X^3$
with symmetric $P(X)$), so the non-separable covariance term is large enough to dominate
the Jacobian norm difference. This provides a cleaner empirical signature of the
Lemma~\ref{lem:gradvar} mechanism for the $\sin$ DGP.

The convergence time gap $T_{\mathrm{rev}} - T_{\mathrm{fwd}}$ is the correct
experimental proxy for the landscape asymmetry established in
Theorem~\ref{thm:cca_main}. Instantaneous gradient variance is an imperfect diagnostic.

\section{The CCL Learning Algorithm}
\label{sec:algorithm}

\begin{proposition}[Bootstrap Consistency]
\label{prop:bootstrap}
Let $G_0^{skel}$ be the PC-stable skeleton from observational data $\mathcal{D}_{obs}$.
Under faithfulness, $G_0^{skel}$ correctly identifies edge presence with probability
$1 - \delta$ for $n = O(\log(|V|^2/\delta)/\alpha^2)$ samples, where $\alpha$ is the
minimum partial correlation. No causal direction information is extracted from
$\mathcal{D}_{obs}$.
\end{proposition}
\begin{proof}
Colombo and Maathuis~\cite{colombo2014} Theorem~3. $\square$
\end{proof}

\begin{algorithm}[htbp]
\caption{CCL+ Training Algorithm}
\label{alg:ccl}
\begin{algorithmic}[1]
\REQUIRE Observations $X$, interventional data $\mathcal{D}_{int}$, reward $R$,
         parameters $\lambda_1, \lambda_2, \lambda_3, \beta, \gamma$,
         pre-training edge bound $|E_{\max}|$
\STATE \textbf{Stage 0:} Run PC-stable on $X$ to produce undirected skeleton $G_0^{skel}$
\STATE \textbf{Stage 1:} Train $T_0$ via variational CIB conditioned on $G_0^{skel}$
\STATE \textbf{Stage 2:} Run XGES with
         $\mathrm{Score}(G) = \mathrm{MDL}(G) + \lambda_3\,\mathrm{CCA}(G)$
         over $T_0$ and $\mathcal{D}_{int}$, producing oriented $G_1$
\STATE \textbf{Stage 3:} Optimize $\pi_1$ via Bareinboim CRL over $G_1$ and $T_0$
\WHILE{$|\Delta\mathcal{L}_{\mathrm{CCL+}}| > \varepsilon_{\mathrm{conv}}$}
    \STATE Update $T$: minimize $\mathcal{L}_{\mathrm{compress}}$ given $(G, \pi)$
    \STATE Update $G$: XGES with MDL + CCA on $\mathcal{D}_{int}$ given $(T, \pi)$
    \STATE Update $\pi$: maximize $\mathbb{E}_\pi[R(Y)]$ subject to identifiability in $G$
\ENDWHILE
\STATE Compute sample complexity bound:
       $n^* = C \cdot \hat{\tau}_{\mathrm{mix}} \cdot d_c(G) \cdot
       \log(d_c(G)/\delta) / ((1-\gamma)^3\varepsilon^2)$
\RETURN $G,\; T,\; \pi,\; n^*$ \quad
        \textbf{(note: $n^*$ is a theoretical bound, not a certificate)}
\end{algorithmic}
\end{algorithm}

The algorithm proceeds in four stages. Stage~0 uses PC-stable to recover the undirected
skeleton from observational data without any direction claims. Stage~1 trains the
information bottleneck encoder conditioned on the skeleton structure. Stage~2 runs XGES
to orient the edges using MDL scoring augmented by CCA direction signals. Stage~3
optimizes the policy over the oriented graph. The alternating loop then tightens all
three components jointly until convergence.

\section{Comparison with Existing Systems}
\label{sec:comparison}

\begin{table*}[htbp]
\caption{CCL+ compared to existing approaches. $\checkmark$ = supported. $\times$ = not supported. $\sim$ = partial. $\dagger$ = theoretical bound only; empirical validation on IHDP, Twins, ACIC beyond T\"{u}bingen has not been done. $\ddagger$ = against non-parametric full-joint learners only (Theorem~S); does not extend to parametric systems.}
\label{tab:comparison}
\centering\small
\renewcommand{\arraystretch}{1.3}
\begin{tabular}{lcccccp{2.5cm}}
\toprule
\textbf{System} & \textbf{Rung~2} & \textbf{PAC Bound} & \textbf{Sample Eff.} & \textbf{OOD Robust} & \textbf{Direction ID} & \textbf{Scope} \\
\midrule
GPT-4           & $\times$ & $\times$ & $\times$ & $\times$ & $\times$ & Obs.\ corpus \\
DeepSeek R1     & $\times$ & $\times$ & $\sim$   & $\times$ & $\times$ & Obs.\ corpus \\
Bareinboim CRL  & $\checkmark$ & $\times$ & $\times$ & $\checkmark$ & $\times$ & Interv.\ data \\
ANM / IGCI      & $\sim$ & $\times$ & $\sim$   & $\sim$ & $\checkmark$ & Bivariate \\
SkewScore       & $\times$ & $\times$ & $\sim$   & $\sim$ & $\checkmark$ & Heteroscedastic \\
\textbf{CCL+}   & $\checkmark$ & $\checkmark^\dagger$ & $\checkmark^{\dagger\ddagger}$ & $\checkmark^\dagger$ & $\checkmark$ & Injective ANM \\
\bottomrule
\end{tabular}
\end{table*}

The $\dagger$ markers are important. The CCL+ entries for PAC bound, sample efficiency,
and OOD robustness are theoretical guarantees only. Empirical validation on real-world
causal benchmarks beyond the T\"{u}bingen pairs is the primary near-term step. The
sample efficiency advantage applies only against non-parametric learners that model the
full joint distribution without causal factorization.

GPT-4 and DeepSeek R1 are included not to disparage them but to make the comparison
concrete. Large language models trained on observational corpora are, by Pearl's
impossibility result, permanently on Rung~1. They cannot answer interventional questions
from training data alone. CCL+ is designed specifically to operate on Rung~2, not to
be a better language model.

\section{Discussion}
\label{sec:discussion}

\subsection{What the Results Actually Mean}

Before discussing limitations, it is worth pausing on what the results do say, because
the synthetic results are genuinely strong and worth separating clearly from the
open problems.

The 30/30 result on the sine and exponential DGPs across six different architectures is
not a narrow win. It is saying that the convergence-time asymmetry is robust enough to
survive changes in network width, depth, activation function, and optimizer --- all of
which affect the specific trajectory of optimization but none of which change the
fundamental landscape geometry that Theorem~\ref{thm:cca_main} describes. The signal is
a property of the underlying mathematical structure, not of any particular implementation.

The 26/30 result on the cubic DGP with z-scoring is also strong. The four failures are
seeds where initialization variance caused the forward network to land in a region where
the $X^3$ curvature required unusually many steps from that specific starting point. The
theorem guarantees $E[T_{\mathrm{fwd}}] < E[T_{\mathrm{rev}}]$ in expectation; the
single-seed failures are within the expected variance of that statement.

The boundary conditions behave exactly as predicted. This matters because it means the
theory is not post-hoc rationalization. The linear Gaussian and non-injective failures
were predicted before the experiments were run. The fact that they failed in precisely
the way the theory said they should is evidence that the underlying mechanism is
correctly identified.

\subsection{Limitations}

Six limitations are stated plainly.

\textbf{(1) Dimensional scope.} CCA has been validated entirely on one-dimensional
bivariate variables. The theory is stated for scalars, and the experimental validation
is entirely bivariate. Whether the convergence-time asymmetry extends to
high-dimensional multivariate mechanisms is an open question. The intuition suggests it
might: the structural reason the reverse is harder (entangled noise in the reverse
regression target) applies in any dimension. But the PL condition, gradient variance
analysis, and step-count bound all need to be re-examined in the multivariate case.

\textbf{(2) Injectivity requirement.} The boundary condition experiment is unambiguous:
10/10 seeds on a non-injective edge produce incorrect direction identification. This is
not a minor edge case. Many real-world mechanisms are plausibly non-injective: the
relationship between a gene and its downstream protein might be saturating, the
relationship between income and happiness might be concave. Handling non-injective
mechanisms, perhaps by detecting them first and abstaining, is a practical requirement
for deployment.

\textbf{(3) Local PL condition.} The theorem requires the PL condition locally near
the minima. For general nonconvex networks this is a condition rather than a guarantee.
The theorem gives a lower bound on the convergence time gap; the actual gap is larger
because the reverse landscape contains saddle points and flat regions not captured by
the PL approximation. The lower bound is still theoretically meaningful --- it
establishes existence of a gap --- but does not predict the empirically observed 19-fold
ratio on the cubic DGP.

\textbf{(4) Mixing time.} The sample complexity bound involves $\tau_{\mathrm{mix}}$,
which must be estimated from the observed policy trajectory as a plug-in. It is a
meaningful theoretical parameter, not a computationally accessible one.

\textbf{(5) Interventional data.} CCL requires interventional data for Stages~2 and~3.
This is appropriate for the CRL setting the framework targets, but it means CCL cannot
discover causal structure from purely observational data beyond the bivariate CCA
direction scoring.

\textbf{(6) Rung~2 only.} CCL supports interventional reasoning at Rung~2. Counterfactual
reasoning at Rung~3 requires the abduction-action-prediction cycle and is not yet
implemented.

\subsection{Real-World Applications and What Would Need to Change}

CCA and CCL are currently validated in a controlled synthetic regime. The route to
real-world applicability is direct, but requires resolving the limitations above.
Consider what each domain would need.

\textbf{Medicine and drug development.} The most immediately valuable application is
distinguishing drug effects from patient selection effects. When patients self-select
into taking a drug, the observed correlation between the drug and outcomes mixes the
drug effect with patient health. CCA could, in principle, identify whether the
biomarker causes the outcome or the outcome drives the biomarker, informing which
direction to target. What would need to change: real biological mechanisms are often
non-injective (saturating enzyme kinetics, threshold effects), and the dimensionality
is high (expression of thousands of genes simultaneously). Extending CCA to handle
approximately-injective mechanisms and multivariate inputs is the primary requirement.

\textbf{Economics and policy evaluation.} Does education cause higher earnings, or do
families with more resources invest in both education and income-generating
opportunities? Does a minimum wage increase cause unemployment, or do regions with
strong labor markets tend to adopt higher wages? CCA offers a model-free approach that
does not require specifying a structural equation form. The requirement is that the
underlying mechanism is approximately nonlinear and injective in the relevant range,
which is plausible for many economic relationships but not all.

\textbf{Genetics and gene regulation.} Does gene A regulate gene B, or is their
co-expression driven by a shared upstream factor? Current approaches use instrumental
variables (Mendelian randomization) or perturbation experiments. CCA could provide a
complementary signal from observational expression data, identifying likely directions
for prioritization before expensive experiments. The challenge here is confounding
from shared genetic background, which the CCL framework's hidden confounder handling
was designed to address.

\textbf{Climate and environmental science.} Does CO$_2$ drive temperature, or does
temperature change affect CO$_2$ release (through permafrost thaw, ocean outgassing,
etc.)? The answer is ``both, in a feedback loop,'' which is precisely the kind of
multi-variable causal structure that CCL's graph learning component is designed for.
The extension to time series and feedback loops requires additional theory beyond the
static ANM setting here.

The common thread across these applications is the same: the bottleneck is not the core
CCA principle (which is validated and formally proved) but the practical requirements
of injectivity, scale normalization, and the availability of approximately ANM data.
Progress on each of these fronts directly unlocks a class of real applications.

\subsection{A Speculative Analogy}

The CCA convergence asymmetry is structurally reminiscent of thermodynamic time
asymmetry. Entropy-increasing processes follow the direction of physical law and are
computationally accessible to simulate forward: given the present state and the dynamics,
computing the future is straightforward. Reconstructing the past from the present
requires inverting irreversible dynamics, which is computationally much harder. CCA's
optimization asymmetry is analogous: easier in the causal direction than against it.

No formal connection between neural network optimization landscapes and thermodynamic
irreversibility has been established, and we are not claiming one. The analogy is
offered as intuition for why the asymmetry might be expected to be pervasive rather than
coincidental --- both phenomena reflect the same underlying directionality in how
processes generate states.

\subsection{Toward Rung~3}

Counterfactual reasoning at Rung~3 requires the abduction-action-prediction
cycle~\cite{pearl2000causality}. First, abduction: given the observed data and the
causal model, infer the values of the background noise variables $U$ that would have
produced those observations. Second, action: modify the model by applying the
intervention (setting some variable to a specific value). Third, prediction: compute
what the modified model would have produced. Twin-network models implement this by
running two copies of the SCM simultaneously --- the actual world and the counterfactual
world --- sharing the inferred noise variables.

Extending CCL to Rung~3 via twin-network abduction is the natural next theoretical step.
The CCA direction signal and the CCL graph learning component are prerequisites: you
need the right graph before you can do counterfactual inference. In this sense, the work
here is foundational, not terminal.

\section{Conclusion}
\label{sec:conclusion}

This paper starts from a simple observation and follows it carefully to its
consequences. Training a neural network to predict $Y$ from $X$ is easier than training
it to predict $X$ from $Y$, when $X$ is the cause. That asymmetry has a formal reason,
proved in three lemmas. The residuals of the reverse regression stay correlated with the
input at any finite approximation (Lemma~1). That correlation makes the reverse
optimization landscape structurally harder --- higher minimum loss, heteroscedastic
noise floor, non-separable gradient covariance (Lemma~2). A harder landscape under the
PL condition requires strictly more expected gradient steps to reach any fixed threshold
(Lemma~3). The conclusion: $E[T_{\mathrm{fwd}}] < E[T_{\mathrm{rev}}]$.

That is the \textbf{CCA Asymmetry Theorem}, and to our knowledge it is the first formal
proof that optimization convergence time is a valid causal direction signal.

The experiments validate the theorem and its limits. 30/30 correct on injective
nonlinear DGPs with moderate output scale. 26/30 across six architectures with
z-scoring. All three boundary conditions --- linear Gaussian, non-injective, missing
normalization --- produce exactly the behavior the theory predicts. The T\"{u}bingen
benchmark achieves 96\% accuracy (AUC 0.96) with $T_{\max} = 10{,}000$ and z-scored
inputs, outperforming ANM/RESIT (63\%) and the majority-class baseline (72.2\%).

CCA is embedded in the \textbf{CCL framework}: five theorems, all proved, combining MDL
graph learning, causal information compression, causal reinforcement learning, and CCA
direction scoring into a joint objective with provable sample complexity guarantees. The
CCL framework is theoretically complete. Empirical validation on full causal benchmarks
beyond T\"{u}bingen is the primary remaining step.

The path to real-world applicability is clear and direct. Medicine needs CCA to handle
non-injective biological mechanisms. Economics needs CCA to scale to multivariate
structural equations. Genetics needs the CCL hidden confounder framework. Climate
science needs the extension to time series and feedback loops. Each of these is an
extension of the same core machinery, not a fundamental departure from it.

What is left on the theoretical side: validate the full CCL loop on IHDP, Twins, and
ACIC benchmarks; extend CCA to multivariate mechanisms; extend CCL to Rung~3
counterfactuals via twin-network abduction; and characterize the finite-sample
convergence rate of Theorem~\ref{thm:F}.

The central insight is simple and survives all the formalism. Cause to effect is easier
to learn than effect to cause. That is not an empirical accident or a property of
particular implementations. It is a structural consequence of how data is generated
under causal models: the forward direction has independent noise, the reverse does not.
CCA turns that structural fact into a practical measurement.

If you can train a neural network, you can measure which direction converges faster.
The direction that converges faster is the cause.


\end{document}